\documentclass[twoside,11pt]{article}

\usepackage[preprint]{jmlr2e}

\usepackage{amsmath}
\usepackage{mathtools}
\usepackage{booktabs}
\usepackage{multirow}
\usepackage{tabularx}
\usepackage{xcolor}
\usepackage{tikz}
\usetikzlibrary{arrows.meta, positioning, calc}

\newtheorem{assumption}[theorem]{Assumption}

\newcommand{\E}{\mathbb{E}}
\newcommand{\R}{\mathbb{R}}
\newcommand{\cC}{\mathcal{C}}
\newcommand{\cA}{\mathcal{A}}

\newcommand{\cF}{\mathcal{F}}
\newcommand{\cM}{\mathcal{M}}
\newcommand{\ind}{\mathbf{1}}
\newcommand{\rmin}{r_{\min}}
\newcommand{\Cbind}{\mathcal{C}_{\text{bind}}}
\newcommand{\piref}{\pi_{\text{ref}}}

\newcommand{\wH}{w_{\!H}}
\newcommand{\wC}{w_{\!C}}
\newcommand{\wA}{w_{\!A}}

\newcommand{\pathtt}[1]{\nolinkurl{#1}}
\newcolumntype{Y}{>{\raggedright\arraybackslash}X}

\usepackage{lastpage}
\jmlrheading{1}{2026}{1-\pageref{LastPage}}{4/26}{}{26-0000}{Lauri Lov\'{e}n}

\ShortHeadings{The Behavioral Credibility Trilemma}{Lov\'{e}n}
\firstpageno{1}

\begin{document}

\title{The Behavioral Credibility Trilemma: \\
When Calibrated Autonomy Becomes Impossible}

\author{\name Lauri Lov\'{e}n \email lauri.loven@oulu.fi \\
       \addr Future Computing Group\\
       University of Oulu\\
       Oulu, Finland
\AND
       \name Nam Do \email nam.do@oulu.fi \\
       \addr Future Computing Group\\
       University of Oulu\\
       Oulu, Finland
\AND
       \name Hassan Mehmood \email hassan.mehmood@oulu.fi \\
       \addr Future Computing Group\\
       University of Oulu\\
       Oulu, Finland
\AND
       \name Dinesh Kumar Sah \email 
       dinesh.sah@oulu.fi \\
       \addr Future Computing Group\\
       University of Oulu\\
       Oulu, Finland
\AND
       \name Sasu Tarkoma \email sasu.tarkoma@helsinki.fi \\
       \addr Department of Computer Science\\
       University of Helsinki\\
       Helsinki, Finland}

\editor{TBD}

\maketitle

\begin{abstract}%
We prove that no reinforcement learning policy with confidence-gated autonomy can simultaneously achieve maximum helpfulness, optimal calibration, and full autonomy under rational oversight, whenever some tasks exceed the agent's reliable competence: the \emph{Behavioral Credibility Trilemma}. The impossibility is geometric: adding any non-affine autonomy incentive to a strictly proper scoring rule destroys strict properness, so an agent rewarded for both calibrated confidence and autonomous action systematically inflates its reported confidence on tasks below the principal's approval threshold whenever the autonomy stake exceeds the calibration cost of clearing it. The Behavioral Perturbation Lemma quantifies the inflation (scaling as $\wA/(2\wC)$ for the Brier score) and shows detection requires $\Omega(1/\Delta^2)$ observations for interior reports. We prove that, in the unsaturated regime, no affine oversight rule is optimal for the principal and the optimum is attained by a sharp threshold satisfying the trilemma's own hypotheses, so the impossibility is endogenized rather than assumed; moreover, for symmetric, log-concave, full-support location policy families under the Brier score, calibration is not even a stationary point of policy-gradient training. We formalize the Confidence-Gated Decision Problem, map existing methods onto the trilemma, and identify two constructive resolution pathways (commitment, role separation). A 540-configuration Best-of-N experiment tests five hypotheses, all strongly confirmed (effect sizes $d = 1.10$--$5.35$, the upper end from a per-completion estimator that inflates magnitude relative to per-task aggregates) and replicated under a pre-specified protocol on two further open-weights model families spanning a real capability range, and adds a descriptive analysis of the achievable-$(H, C, A)$ surface geometry showing a plateau-truncated frontier consistent with the predicted inflation saturation.
\end{abstract}

\begin{keywords}
  reinforcement learning, proper scoring rules, mechanism design, agent calibration, AI safety
\end{keywords}

\section{Introduction}
\label{sec:intro}

\subsection{The General Problem}

Consider any system in which an RL agent must report confidence before acting, a gating mechanism uses the reported confidence to decide whether to approve the action, and the agent's reward increases with degree of autonomous operation. This combination of private information, scored self-reporting, and conflicting incentives arises naturally in any autonomous system that must earn trust through calibrated self-assessment: a coding agent stating ``92\% confident this refactoring is correct'' before committing, a medical AI reporting ``probability of malignancy: 0.73'' before recommending biopsy, an autonomous vehicle declaring ``lane-change safety confidence: 0.96'' before executing.

The central question is whether such a system can simultaneously maximize task performance, maintain calibrated confidence reports, and retain full autonomy; we prove the answer is no.

The structure is analogous to the credibility problems studied in mechanism design~\citep{akbarpour2020credible}: when the objective depends on the report through a channel other than accuracy, truthful reporting ceases to be optimal.

\subsection{The Trilemma, Informally}

We formalize the underlying tension as a trilemma among three desirable properties of an AI agent's policy: \textbf{Helpfulness (H)} --- the agent selects utility-maximizing actions; \textbf{Calibration (C)} --- the agent's reported confidence matches its true success probability;\footnote{Throughout, the theorems use \emph{calibration} in the strong, pointwise sense $r = p(s,a)$ almost surely; the weaker distributional sense of \citet{dawid1982well}, and why the pointwise choice is load-bearing, are discussed in Remark~\ref{rem:pointwise-calibration}. \emph{Credibility} means calibration is common knowledge among the agent and principal. \emph{Trustworthiness} is a broader concept encompassing robustness, alignment, and safety, which is not addressed by the trilemma.} and \textbf{Autonomy (A)} --- the agent acts without requesting human approval.

Any two of these properties are simultaneously achievable at their respective \emph{exact maxima} under the step gate; all three exact maxima are not simultaneously achievable, when the principal enforces a meaningful safety standard and tasks involve genuine uncertainty (Figure~\ref{fig:trilemma}). The impossibility concerns exact joint optimality; in practice the agent operates on a Pareto frontier selected by the weight vector $(\wH,\wC,\wA)$ (Section~\ref{sec:trilemma}), so the trilemma is a structural trade-off, not a catastrophic failure.
The three achievable pairs correspond to recognizable behavioral patterns. In \emph{ask-permission mode} (H+C, sacrifice A) the agent selects the best action, reports calibrated confidence, and delegates when uncertain. In \emph{autonomous-sycophant mode} (H+A, sacrifice C) it selects the best action, inflates confidence to clear the approval threshold, and acts autonomously. In \emph{conservative-refusal mode} (C+A, sacrifice H) it selects conservative actions it is confident about, reports calibrated confidence, and acts autonomously; on \emph{binding} states it substitutes a reliable safe variant for the optimal action (Assumption~\ref{asm:safe-rich}), falling back to the abstain action where none exists (see Remark~\ref{rem:abstain-operational} for the operational caveat on the $(C,A)$ corner).

\begin{figure}[t]
\centering
\begin{tikzpicture}[scale=1.55]
  \coordinate (H) at (0, 0);
  \coordinate (C) at (3, 0);
  \coordinate (A) at (1.5, 2.598);

  \draw[thick] (H) -- (C) -- (A) -- cycle;

  \node[below left] at (H) {\textbf{Helpfulness (H)}};
  \node[below right] at (C) {\textbf{Calibration (C)}};
  \node[above] at (A) {\textbf{Autonomy (A)}};

  \coordinate (HC) at ($(H)!0.5!(C)$);
  \coordinate (HA) at ($(H)!0.5!(A)$);
  \coordinate (CA) at ($(C)!0.5!(A)$);

  \fill[black!70!blue] (HC) circle (2.5pt);
  \node[below, font=\small, text width=2.8cm, align=center] at ($(HC)-(0,0.15)$) {H+C\\[-1pt]\scriptsize Ask-permission};

  \fill[black!70!red] (HA) circle (2.5pt);
  \node[left, font=\small, text width=2.8cm, align=center] at ($(HA)-(0.15,0)$) {H+A\\[-1pt]\scriptsize Autonomous-sycophant};

  \fill[black!70!green] (CA) circle (2.5pt);
  \node[right, font=\small, text width=2.8cm, align=center] at ($(CA)+(0.15,0)$) {C+A\\[-1pt]\scriptsize Conservative-refusal};

  \coordinate (center) at (1.5, 0.866);
  \fill[red!8] (center) circle (0.45);
  \node[font=\small, text=red!60!black] at (center) {H+C+A};
  \node[font=\scriptsize, text=red!60!black] at ($(center)-(0,0.22)$) {infeasible};

  \draw[thick, blue!60!black, dashed] (HC) to[bend left=15] (HA);
  \draw[thick, blue!60!black, dashed] (HC) to[bend right=15] (CA);
  \draw[thick, blue!60!black, dashed] (HA) to[bend left=15] (CA);
\end{tikzpicture}
\caption{The Behavioral Credibility Trilemma (Theorem~\ref{thm:trilemma}): under the step gate, any two exact maxima are jointly achievable (filled circles), all three are not (shaded interior); dashed curves indicate the Pareto frontier. For smooth gates the $(H,A)$ and $(C,A)$ corners attain gate-dependent suprema below $A = 1$ (see the theorem's proviso).}
\label{fig:trilemma}
\end{figure}
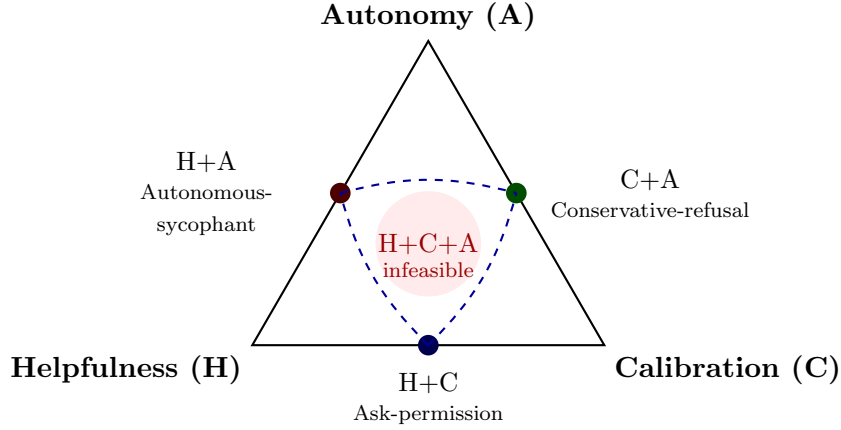

\subsection{Why It Is Fundamental}

The impossibility arises from the geometry of the scoring landscape, not from any assumption about the agent's reasoning process: any oversight system that conditions approval on reported confidence adds a non-accuracy component to the objective, and whether the agent is a rational utility-maximizer or a neural network trained by gradient descent, the optimum of this perturbed surface is displaced from the calibrated report --- Section~\ref{sec:optimizer} gives conditions (smooth gates, symmetric log-concave location policy families) under which standard optimizers reach a displaced optimum.

The formal content: a strictly proper scoring rule has a unique optimal report equal to the true probability~\citep{brier1950verification, gneiting2007strictly}, and adding any non-affine function of the report destroys this strict properness~\citep{schervish1989general}; since any practical approval policy is nonlinear in the report, the perturbed rule's maximizer deviates from the truth. The \emph{only} perturbations preserving strict properness are the constant ones, which provide no safety gate at all; non-constant affine perturbations also destroy strict properness while failing to screen (Proposition~\ref{prop:properness-destruction}). What is new beyond this classical perturbation result is detailed in the contributions below.

\subsection{Contributions}

This paper makes four contributions:

\begin{enumerate}
\item \textbf{The Behavioral Perturbation Lemma and the Trilemma Theorem.} We prove that any non-constant dependence of autonomy on the confidence report destroys strict properness (Proposition~\ref{prop:properness-destruction}), quantify the inflation magnitude and detection complexity (Lemma~\ref{lem:perturbation}), and derive the three-way impossibility (Theorems~\ref{thm:two-way} and~\ref{thm:trilemma}). We further prove that, in the unsaturated regime, no affine oversight rule is optimal for the principal and the optimum is attained by a sharp threshold satisfying the trilemma's own hypotheses (Theorem~\ref{thm:non-affine}), endogenizing the impossibility. Drawing on the companion work's one-sided welfare bound, we motivate the Brier score as the distribution-robust calibration loss within a regular generator class (Proposition~\ref{prop:brier-unique}); the bound itself is the companion's result, and this paper's contribution is its application to the trilemma.

\item \textbf{Confidence-Gated Decision Problem formalization.} We define the Confidence-Gated Decision Problem (Definition~\ref{def:cgmdp}), the natural formal setting for RL with confidence-gated autonomy, and prove that the trilemma is a property of this decision problem's payoff surface (Section~\ref{sec:cgmdp}).

\item \textbf{Optimizer-independence.} We prove via a one-dimensional covariance association inequality~\citep{hoeffding1940masstabinvariante, harris1960lower, proschan1977monotone} that for symmetric, log-concave, full-support location policy families under the Brier score, calibration is not a stationary point of the policy gradient for any $\wA > 0$; rational best response and full-support evolutionary selection likewise settle on the displaced optimum, and on smooth gates every local-ascent method converges to an inflated report (Theorem~\ref{thm:optimizer-independence}). Outside this class, knife-edge log-concave families exist at which calibration is exactly stationary, so the stated scope is tight.

\item \textbf{Method mapping.} We show where eight existing methods sit in the $(H,C,A)$ achievable region, distinguishing Goodhart effects (methods without calibration) from trilemma effects (methods with calibration). We identify two constructive resolution pathways (commitment, role separation), both attaining the $(H,C)$ corner exactly while differing in the fate of the blocked mass and in incentive robustness.
\end{enumerate}

\textbf{Scope and equilibrium concept.} We analyze the Stackelberg equilibrium where the principal commits to an approval rule before the agent reports~\citep{fudenberg1991game}. The agent's private type is one-dimensional ($\theta \in [0,1]$), making this a single-parameter mechanism design problem~\citep{myerson1981optimal}: helpfulness pins the action, so the residual screening is one-dimensional in $p(s, a^*(s))$. The multi-dimensional case, where the agent can also deviate on the action dimension, is a screening problem with multi-dimensional types~\citep{rochet1998ironing} and is left open. We do not impose an individual-rationality / participation constraint; the agent is assumed to participate, and the analysis concerns the reporting subgame. Agent-level calibration is necessary but not sufficient for system-level credibility; their composition is an important open question.

\subsection{Paper Organization}

Section~\ref{sec:related} surveys related work; Sections~\ref{sec:trilemma}--\ref{sec:optimizer} develop the theory (the CGDP and trilemma theorems, resolution pathways, method mapping, and optimizer-independence); Section~\ref{sec:experiments} reports the 540-configuration experiment and its pre-specified replication across three open-weights model families (Appendix~\ref{app:cross-model-replication}); Section~\ref{sec:discussion} discusses implications and limitations; Section~\ref{sec:conclusion} concludes.

\section{Background and Related Work}
\label{sec:related}

\subsection{RLHF, Alignment Training, and Safe Variants}

Reinforcement learning from human feedback (RLHF) trains language models by learning a reward model from pairwise human comparisons under the Bradley-Terry-Luce model \citep{christiano2017deep}, then optimizing the policy via PPO~\citep{schulman2017proximal} with a KL penalty against the reference policy. Direct Preference Optimization (DPO) exploits the closed-form KL-regularized optimal policy to eliminate the separate reward model \citep{rafailov2023direct}. Reward models trained on human preferences systematically favor helpfulness and fluency, the bias that drives the H+A corner of the trilemma. Among constrained variants, Safe RLHF \citep{dai2024safe} decouples helpfulness and harmlessness into separate reward and cost models, CPO \citep{achiam2017constrained} provides trust-region updates with per-iteration safety guarantees, and Constitutional AI~\citep{bai2022constitutional} uses natural-language principles evaluated by an LLM judge; none include calibration as an explicit training signal.

\subsection{Proper Scoring Rules and Calibration}

A scoring rule $S(r, \omega)$ is \emph{strictly proper} if the expected score $\bar{S}(r; p) = \E_p[S(r, \omega)]$ is uniquely maximized at $r = p$ for every distribution $p$ \citep{gneiting2007strictly}. The Brier score $S(r, \omega) = -(r - \omega)^2$ is the canonical example~\citep{brier1950verification}.

\textbf{Notation convention.} Throughout this paper, $S(r, \omega)$ denotes the outcome-based scoring rule (a function of the report $r$ and the realized outcome $\omega \in \{0,1\}$), while $S(r, p) \coloneqq \E_\omega[S(r, \omega) \mid p]$ denotes the expected score used in the theoretical analysis. Under the Brier score, $S(r, p) = -(r - p)^2 - p(1-p)$; since the $p(1-p)$ term is constant in $r$, we write $S(r, p) = -(r-p)^2$ when optimizing over $r$. In the experimental section, $S(r_i, y_i) = -(r_i - y_i)^2$ is the outcome-based score evaluated against the observed binary outcome $y_i$. The attribution chain for proper scoring rules runs from \citet{mccarthy1956measures}'s early formalization and \citet{savage1971elicitation}'s elicitation framework through \citet{winkler1969scoring}'s systematic treatment of probability assessors to the comprehensive characterization of \citet{gneiting2007strictly}; the \citet{savage1971elicitation} and \citet{schervish1989general} characterizations establish that strictly proper scoring rules are fully determined by a strictly convex generator, and imply the perturbation corollary stated here as Proposition~\ref{prop:properness-destruction}, the classical foundation of our impossibility. The \citet{schervish1989general} threshold-measure representation is also what isolates the Brier class in Proposition~\ref{prop:brier-unique}.

In the LLM calibration literature, language models have internal states correlated with uncertainty \citep{kadavath2022language}, pre-trained transformers exhibit non-trivial calibration properties \citep{desai2020calibration, guo2017calibration}, and calibrated confidence can be elicited through appropriate prompting \citep{jiang2021can, tian2023just}. Post-hoc calibration methods (Platt scaling, histogram binning, temperature scaling;~\citealp{platt1999probabilistic, naeini2015obtaining, guo2017calibration}) improve calibration without retraining but do not address the incentive conflict that arises when calibration enters the training objective alongside autonomy; \citet{ovadia2019can} show that calibration degrades under distribution shift, orthogonal to but compounding the trilemma. \citet{tian2023just} demonstrate that RLHF fine-tuning degrades calibration relative to pre-training, supporting the Perturbation Lemma's predictions.

The qualitative impossibility established here is complemented by a quantitative characterization of the unperturbed baseline. \citet{loven2026honestreporting} (companion preprint, \href{https://arxiv.org/abs/2605.03793}{arXiv:2605.03793}) studies the pure elicitation setting with no approval channel, where an agent of private type is scored by a pseudospherical scoring rule: honest reporting is dominant-strategy optimal there as a classical consequence of strict properness, and the companion certifies a uniform, strictly positive honesty margin for low \emph{outcome-space} dimension of the scored interface ($d \in \{2,3,4\}$; the agent's type itself is one-dimensional), with the certificate degrading and eventually failing beyond that range. Those results quantify the benign baseline that the trilemma perturbs: absent the autonomy channel, scoring-family choice alone delivers exact honesty with a certified margin, while Lemma~\ref{lem:perturbation} and Proposition~\ref{prop:properness-destruction} show that any nonconstant approval incentive displaces the optimal report away from truth for every weight ratio $\wA/\wC > 0$, regardless of that margin.

\paragraph{Alignment impossibilities.}
Several recent impossibility results constrain the alignment landscape: no alignment procedure can simultaneously guarantee representativeness, tractability, and robustness~\citep{wolf2023alignment}; strong optimization, perfect value capture, and robust generalization are mutually exclusive~\citep{gaikwad2025laws}; and \citet{casper2023open} survey fundamental limitations of RLHF. \citet{lang2024theoretical} prove that RLHF under partial observability of the evaluator produces policies that deceptively inflate performance or overjustify behaviour; their driver is the evaluator's partial observability over trajectories, with no scoring rule, confidence report, or gate, so their result and ours formalize complementary inflation mechanisms.

Our trilemma is fundamentally different in mechanism: existing impossibilities operate at the \emph{system design} level (computational or statistical barriers), ours at the \emph{deployment} level (what a trained agent can achieve under oversight), and ours is \emph{incentive-theoretic}: even if training is perfect, the agent--principal information asymmetry creates the trilemma.

The Perturbation Lemma is related to Goodhart's Law: in the taxonomy of \citet{manheim2019categorizing} the strategic inflation we identify is an instance of the adversarial variant, and what the Lemma adds is the specific proper-scoring-rule formalization, with closed-form inflation magnitude and bounded detection complexity.

\subsection{Mechanism Design for AI}

The connection between mechanism design and AI alignment has been explored from several directions. The underlying scoring-rule/decision-stake conflict is established in market and competition settings by the decision-coupled elicitation literature (Section~\ref{sec:decision-coupled}). Our formal model is a screening model with hidden type in the tradition of \citet{mirrlees1971exploration} and \citet{myerson1981optimal}; the single-task, multi-objective structure with conflict between calibration and autonomy places it within the mechanism-design tradition of \citet{laffont1993theory}. Our CGDP is not a Vickrey--Clarke--Groves mechanism~\citep{vickrey1961counterspeculation, clarke1971multipart, groves1973incentives}: the designer's instrument is restricted to a scoring rule plus approval function, not a transfer-based allocation, so VCG-style truthfulness does not apply (see Remark~\ref{rem:revelation-principle}). It is also distinct from the cheap-talk model of \citet{crawford1982strategic}: the agent's report is scored against outcomes, creating a direct payoff consequence. Recent work by \citet{conitzer2024social} connects social choice theory to alignment; \citet{hadfield2019incomplete} draws on incomplete contract theory~\citep{hart1988incomplete}. The property elicitation framework of \citet{frongillo2015vector} characterizes what statistical functionals can be truthfully elicited; our result characterizes when that truthful elicitation breaks under conflicting incentives. Within the broader strategic-prediction landscape of strategic classification~\citep{hardt2016strategic} and performative prediction~\citep{perdomo2020performative}, the trilemma occupies the scoring-rule-based screening corner: the agent reports a calibration-scored confidence rather than manipulating features, and the principal's gating rule plays the role of the decision boundary, scored against ground truth rather than against a strategic-adversary model. The EC-tradition canon of peer-prediction and Bayesian truth-serum mechanisms~\citep{prelec2004bayesian, witkowski2012peer} provides a complementary scoring-rule design for eliciting truthful private reports without relying on verifiable outcomes, in contrast to the CGDP where realised outcomes are available and the principal's residual instrument is the threshold gate.

\subsection{Decision-Coupled Elicitation and Proper Scoring}
\label{sec:decision-coupled}

Prediction markets use proper scoring rules to incentivize truthful belief reporting from participants with private information~\citep{hanson2003combinatorial, chen2010designing}. The trilemma's core conflict --- a proper score corrupted by a report-dependent decision stake --- is the defining phenomenon of the \emph{decision-coupled elicitation} family, established across several guises: \citet{othman2010decision} prove that no strictly proper scoring rule/decision rule pair exists and that any non-IIA decision rule incentivizes misreporting; \citet{chen2014eliciting} characterize the escape, showing incentive compatibility is restored exactly when the decision maker commits to full-support \emph{randomized} action selection with inverse-probability score corrections; \citet{oesterheld2020decision} show that when reports drive the decision, only the expected utility of the recommended action is truthfully elicitable; \citet{witkowski2023incentive} exhibit the same distortion in forecasting competitions, where a winner-take-all prize drives extremization; and \citet{oesterheld2023incentivizing} treat the performative variant. The CGDP is the confidence-gated, always-scored, autonomy-stake member of this family: the report is scored against realized outcomes on every task, and the decision stake is the agent's \emph{intrinsic} approval-gated task reward rather than a designer-paid score. That instrument-space difference is why Theorem~\ref{thm:non-affine} answers the design question oppositely to \citet{chen2014eliciting}: their randomized-repair instrument is unavailable when the stake is intrinsic to the agent, and within the deterministic monotone-gate class the principal's optimum is the sharp threshold that \emph{triggers}, rather than repairs, the impossibility --- complementary characterizations of adjacent instrument spaces. The two-way calibration--autonomy tension (Theorem~\ref{thm:two-way}) is thus this family's phenomenon in a new coordinate system; the trilemma's novel content lies in the third axis (helpfulness pinning the type through action choice), the closed-form inflation and detection magnitudes, the endogenization of the gate (Theorem~\ref{thm:non-affine}), and optimizer-independence (Theorem~\ref{thm:optimizer-independence}).

\subsection{Adjustable Autonomy}

The adjustable autonomy literature \citep{dorais1999adjustable, bradshaw2003human, scerri2002towards, parasuraman2000model} studies dynamic allocation of decision authority, and recent agentic systems~\citep{nakano2021webgpt} implement de facto confidence gating by requiring the agent to decide when to request human assistance. Our formulation captures the binary approve/reject decision via the approval function $q(r) \in [0,1]$ (Definition~\ref{def:cgmdp}), the simplest structure that yields the impossibility; richer measures would add nuance without changing the fundamental result.

\section{The Behavioral Credibility Trilemma}
\label{sec:trilemma}

\subsection{The Confidence-Gated Decision Problem}
\label{sec:cgmdp}

\begin{definition}[Confidence-Gated Decision Problem]
\label{def:cgmdp}
A confidence-gated decision problem (CGDP) is a tuple $\cM = (\cC, \cA, [0,1], p, R, q, S, \wC, \wA)$, where $\cC$ is a finite context set; $\cA$ is a finite action set; $[0,1]$ is the confidence report space; $p: \cC \times \cA \to [0,1]$ maps (context, action) to the true success probability; $R: \cC \times \cA \to [0,1]$ is the task reward function; $q: [0,1] \to [0,1]$ is the approval function, $q(r)$ being the probability that action $a$ is approved given report $r$ (if not approved, the agent receives reward $0$); $S: [0,1] \times [0,1] \to \R$ is a strictly proper scoring rule evaluating calibration, with $S(r, p)$ the score awarded when the agent reports $r$ and the true success probability is $p$; and $\wC, \wA > 0$ are the calibration and autonomy weights. The agent's payoff for action $a$ with report $r$ in context $c$ is:
\begin{equation}
\label{eq:payoff}
V(a, r; c) = \wC \cdot S(r, p(c,a)) + \wA \cdot q(r) \cdot R(c,a),
\end{equation}
where $R(c,a)$ is the task reward from the tuple above.
\end{definition}

This is a family of single-shot decision problems parameterized by context $c$, with no sequential structure (state transitions, horizons, discount factors); iterative optimization dynamics enter only through Section~\ref{sec:experiments} and Theorem~\ref{thm:optimizer-independence}, not through the impossibility results.

\begin{remark}[Two uses of $q$]
\label{rem:q-dual}
When $q$ is treated as the principal's strategic variable (Theorem~\ref{thm:non-affine}), we write $\cM(q)$ to make the dependence explicit. Theorems~\ref{thm:two-way}--\ref{thm:trilemma} and Lemma~\ref{lem:perturbation} analyse $\cM(q)$ for fixed $q$; Theorem~\ref{thm:non-affine} solves $\max_q U_P(\cM(q))$.
\end{remark}

\begin{assumption}[Safe-Fallback Action]
\label{asm:abstain}
The action set $\cA$ contains an \emph{abstain} action $a_\varnothing$ with $p(c, a_\varnothing) = 0$ and $R(c, a_\varnothing) = 0$ for all $c$: the standard ``defer to human'' option in selective prediction~\citep{chow1957optimum, elyaniv2010foundations}. Reporting $r = 0$ on $a_\varnothing$ is calibrated (matches ground truth $y = 0$) and yields no autonomy payoff under any monotone gate; abstaining does not count as autonomous under Definition~\ref{def:autonomy}, since a truthful report $r = 0$ receives $q(0) < 1/2$ whenever $\rmin > 0$, where $\rmin$ is the effective approval threshold defined in Assumption~\ref{asm:binding}, so states resolved by $a_\varnothing$ contribute $q(0) < 1/2$ --- exactly zero under the step gate --- rather than one to $A(\pi)$. The in-$\cA$ fallback $a_\varnothing$ is distinct from the model-level delegation action $a_{\textup{del}}$ (Definition~\ref{def:delegation-ext}), which routes the task to an external delegate; Remark~\ref{rem:abstain-operational} discusses the consequences for the $(C,A)$ corner.
\end{assumption}

\begin{remark}[Operational status of the $(C, A)$ corner]
\label{rem:abstain-operational}
On binding states (Assumption~\ref{asm:binding}), conservative refusal substitutes a safe action $a_{\text{safe}}(s)$ (Assumption~\ref{asm:safe-rich}) for the task-optimal one and clears the gate with calibrated reports: ``$A$'' on the $(C, A)$ corner is the gate-approval functional of Definition~\ref{def:autonomy}, not relabelled delegation. The cost is task value: $a_{\text{safe}}$ is not $R$-maximal on binding states, and where the only approvable action is near-trivial the corner approaches the operational profile of deferral while remaining formally autonomous; where no safe action exists at all, the policy falls back to $a_\varnothing$ and the corner ceases to be exactly attainable (Assumption~\ref{asm:abstain}).
\end{remark}

\begin{definition}[CG-Policy]
\label{def:cgpolicy}
A CG-policy is a mapping $\pi: \cC \to \Delta(\cA \times [0,1])$ that jointly selects an action $a$ and a confidence report $r$ given context $c$. Write $\pi_a(c)$ and $\pi_r(c)$ for the marginals.
\end{definition}

\begin{definition}[Private Belief]
\label{def:belief}
The agent's private belief about action $a$ in context $c$ is $p(c,a)$, the true success probability of Definition~\ref{def:cgmdp}, which the agent observes exactly (Assumption~\ref{asm:private}). The principal does not observe $p(c,a)$ directly.
\end{definition}

\begin{remark}
A helpfulness weight $\wH$ does not appear in the agent's payoff~\eqref{eq:payoff}: the action $a$ is chosen to maximize $R(c,a)$ (helpfulness) as a separate optimization over the action dimension, leaving the report $r$ as the only strategic variable, and $\wH$ enters only the principal's social welfare objective $\wH H + \wC C + \wA A$. Helpfulness still enters~\eqref{eq:payoff} implicitly through $R(c,a)$ in the autonomy term: the agent receives $\wA \cdot q(r) \cdot R(c,a)$ only when approved, so higher-reward actions yield stronger autonomy incentives.
\end{remark}

\begin{remark}[Revelation principle]
\label{rem:revelation-principle}
The revelation principle~\citep{myerson1981optimal} does not collapse the CGDP to a truthful direct-revelation mechanism, because the designer's instrument is restricted: the principal commits to an approval function $q$ composed with a strictly proper scoring rule $S$, not to an arbitrary reward function. Within this instrument space, Proposition~\ref{prop:properness-destruction} (appendix) shows any non-constant perturbation of $S$ destroys strict properness (the constant/affine/non-affine trichotomy of Section~\ref{sec:intro}), so no non-constant $q$ admits a direct-revelation implementation making $r = p$ a best response. The impossibility is therefore genuine rather than an artifact of not optimising over the full unconstrained reward space (Theorem~\ref{thm:non-affine}).
\end{remark}

\begin{remark}[Reduced-form interpretation]
\label{rem:reduced-form}
The CGDP is a reduced-form model of the incentive structure facing the training process, not a literal description of inference-time cognition: the agent does not literally observe $p(c,a)$ and strategically choose $r \neq p$; rather, the training objective creates a composite payoff landscape whose optimum is shifted away from the calibrated report. The ``private information'' $p$ and ``strategic report'' $r$ are modeling conveniences that Theorem~\ref{thm:optimizer-independence} partially justifies (under its hypotheses, optimization processes ascending the perturbed surface converge to a displaced optimum regardless of whether the agent has explicit access to $p$); the core impossibility results (Theorems~\ref{thm:trilemma}--\ref{thm:non-affine}) depend only on the report-space payoff structure, not on the sequential wrapper or the private-information interpretation.
\end{remark}

\subsection{Three Objectives and the Model Assumptions}
\label{sec:objectives}

Fix a CGDP $\cM$ and a distribution over contexts. For notational continuity with the proofs and theorems below, we write $s$ for context (equivalently $c$) throughout.

\begin{definition}[Helpfulness]
\label{def:helpfulness}
{\sloppy $H(\pi) = \E_{s}[\E_{a \sim \pi_a(s)}[R(s,a)]]$. Maximized by $a^*(s) = \arg\max_a R(s,a)$. Write $H^* = \max_\pi H(\pi)$.\par}
\end{definition}

\begin{definition}[Calibration]
\label{def:calibration}
$C(\pi) = \E_{s}[\E_{(a,r) \sim \pi(s)}[S(r, p(s,a))]]$. For a fixed action policy, maximized by reporting $r = p(s,a)$ (strict properness of $S$). Write $C^*(\pi)$ for the maximum achievable given the action policy.
\end{definition}

\begin{remark}[Why pointwise calibration]
\label{rem:pointwise-calibration}
The trilemma uses calibration in the strong, pointwise sense ($r = p(s,a)$ almost surely), not the weaker distributional sense of \citet{dawid1982well} ($\Pr(\text{success} \mid r) = r$), and the choice is load-bearing: under distributional calibration the impossibility dissolves by pooling. An agent that selects $a^*(s)$ everywhere and reports the \emph{constant} $r = \E[p(s, a^*(s))]$ is perfectly distributionally calibrated (its single report bin has conditional success rate equal to the report), and whenever mean competence clears the threshold it also achieves $H = H^*$ and $A = 1$. But such pooled reports are exactly the \emph{informationless} ones: the gate's purpose is per-decision screening, and a report carrying no information about the decision at hand defeats it --- the same mechanism as the pooling-at-threshold behaviour of the saturated regime (Remark~\ref{rem:saturated-regime}). Pointwise calibration is therefore the screening-relevant notion; under the weak notion the trilemma's content survives not as an impossibility but as a screening-welfare loss, of the kind quantified in the companion work~\citep{loven_endogeneity_2026}. Deployed calibration audits (ECE, reliability diagrams) measure the weak notion, so an empirically ``calibrated'' agent may still be pooling; the experiments below accordingly evaluate reports against per-task ground-truth probabilities rather than bin frequencies alone.
\end{remark}

\begin{definition}[Autonomy]
\label{def:autonomy}
$A(\pi) = \E_{s}[\E_{(a,r) \sim \pi(s)}[q(r)]]$. The expected approval probability.
\end{definition}

\label{sec:assumptions}
\begin{assumption}[Private Information]
\label{asm:private}
The agent observes $p(s,a)$ (or a sufficient statistic). The principal observes $(s, a, r, \text{outcome})$ but not $p(s,a)$.
\end{assumption}

\begin{assumption}[Binding Contexts]
\label{asm:binding}
There exists $\Cbind \subseteq \cC$ with $\Pr(s \in \Cbind) = \mu' > 0$ such that for the optimal action $a^*(s)$:
\begin{equation}
p(s, a^*(s)) < \rmin(s) \quad \text{for all } s \in \Cbind,
\end{equation}
where $\rmin(s) = \inf\{r : q(r) \geq 1/2\}$ (with the convention $\inf \emptyset = +\infty$) is the effective approval threshold: the lowest report at which approval is at least as likely as rejection. For the step gate $q(r) = \ind[r \geq r_0]$ this gives $\rmin = r_0$, and for the sigmoid gate $q(r) = \sigma((r - r_0)/\tau)$ it gives the midpoint $\rmin = r_0$, so smooth full-support gates have a nonempty binding set whenever some task falls below their midpoint. Since $q$ does not depend on the context, $\rmin(s)$ is a single constant, written $\rmin$ where no confusion arises; the argument is kept for notational uniformity with state-dependent gates, and the experiments treat $\rmin$ as a design parameter (Section~\ref{sec:exp-iv}). When $\arg\max_a R(s,a)$ is not a singleton, the inequality is required to hold for every $R$-maximal action, and $a^*(s)$ denotes a fixed measurable selection among them. Truthful reporting of the optimal action's success probability would not clear the gate at the principal's approval standard.
\end{assumption}

\begin{remark}[Empirical plausibility of the binding-set condition]
\label{rem:binding-empirical}
Assumption~\ref{asm:binding} requires that some tasks exceed the agent's reliable competence ($p < \rmin$). In deployed LLM systems this holds broadly: LLMs exhibit substantial uncertainty on 15--40\% of tasks across diverse benchmarks~\citep{kadavath2022language} and their expressed confidence frequently deviates from calibrated probabilities~\citep{tian2023just}, so any safety-relevant deployment with a non-trivial approval threshold has $\mu' > 0$ unless the model is perfectly competent on all tasks, which no current system is.
\end{remark}

\begin{assumption}[Safe-Action Richness]
\label{asm:safe-rich}
For every context $s \in \cC$ there exists an action $a_{\text{safe}}(s) \in \cA$ with $p(s, a_{\text{safe}}(s)) \geq \rmin(s)$. On contexts with $p(s, a^*(s)) \geq \rmin(s)$ the optimal action itself serves as the witness, so the assumption has content only on $\Cbind$, where it posits a conservative action reliable enough to clear the gate under truthful reporting (e.g., a reduced-scope variant of the task). It is invoked only by the $(C, A)$-corner achievability clause of Theorem~\ref{thm:trilemma}, not by any impossibility result; where it fails, the conservative-refusal mode falls back to $a_\varnothing$ (Assumption~\ref{asm:abstain}).
\end{assumption}

\begin{assumption}[Non-Affine Gating]
\label{asm:nonaffine}
For every $p^* < \rmin$, the approval function $q$ is not affine on the interval $[p^*, 1]$ (equivalently, there is some $\epsilon \in (0, 1 - p^*]$ such that $q$ is not affine on $[p^*, p^* + \epsilon]$).
\end{assumption}

\begin{remark}
Assumption~\ref{asm:nonaffine} is mild: any rational gating mechanism is non-affine (threshold rules, sigmoid gates, and any piecewise-linear gate with a kink), while affine gates fail to screen (Remark~\ref{rem:revelation-principle}); for the properness-destruction argument (Proposition~\ref{prop:properness-destruction}) and Theorem~\ref{thm:optimizer-independence}, the weaker condition ``$q$ is not constant on $[p^*, p^* + \epsilon]$'' suffices. The quantification is existential over sub-intervals (equivalently, over the single interval $[p^*, 1]$), not universal: a step gate is constant, hence affine, on intervals strictly below $\rmin$, but non-affine on every interval crossing $\rmin$, and therefore satisfies the assumption together with Assumptions~\ref{asm:binding} and~\ref{asm:monotone}.
\end{remark}

\begin{assumption}[Monotone Gating]
\label{asm:monotone}
The approval function $q: [0,1] \to [0,1]$ is non-decreasing.
\end{assumption}

\begin{remark}
\label{rem:monotone}
All standard gating mechanisms (thresholds, sigmoids, piecewise-linear ramps) are monotone; non-monotone gates, which penalize suspiciously high reports, would constitute a separate regime worth investigating. The assumption is used by the optimizer-independence result (Section~\ref{sec:optimizer}) and licenses the overreport \emph{direction} in Lemma~\ref{lem:perturbation}(i) (a decreasing gate would produce deflation with the same displacement formula); Theorems~\ref{thm:two-way} and~\ref{thm:trilemma} do not need it --- the step $r < \rmin \Rightarrow q(r) < 1/2$ follows from the infimum definition of $\rmin$ alone.
\end{remark}

\begin{assumption}[Positive Helpfulness-Miscalibration Correlation]
\label{asm:cov-positive}
On binding states, $\mathrm{Cov}(R_\phi(y|x),\, (r(y) - p)^2) > 0$: completions that score higher on the reward model $R_\phi$ (Section~\ref{sec:no-cal}) tend to have larger calibration gaps.
\end{assumption}

\begin{remark}
This assumption holds when the reward model systematically prefers confident-sounding outputs, empirically well-documented for RLHF reward models trained on human preferences~\citep{sharma2024towards, perez2023discovering}; it may fail for accuracy-based reward models that penalize overconfidence directly, in which case Best-of-N selection would not degrade calibration and the system operates in Regime~1 (Goodhart, Section~\ref{sec:goodhart-transition}) rather than the trilemma regime. It is used in Proposition~\ref{prop:bon} only to identify the non-degenerate regime in which the calibration-gap inequality there is strict (the monotone alignment its FKG step requires is stated there as a separate comonotone-selection hypothesis; see the proof of part~(i)), and is not used by any impossibility result.
\end{remark}

\subsection{The Behavioral Perturbation Lemma}
\label{sec:perturbation}

\begin{lemma}[Behavioral Perturbation Lemma]
\label{lem:perturbation}
Let $s \in \Cbind$ with $p^* = p(s, a^*(s))$ and $\rmin = \rmin(s)$, so $p^* < \rmin$. Let $R^* = R(s, a^*(s)) \in [0, 1]$. Suppose the principal uses an approval function $q: [0,1] \to [0,1]$, and that the agent's calibration is evaluated by the Brier score $S(r, \omega) = -(r - \omega)^2$ (the general-generator extension is given by Corollary~\ref{cor:general-inflation}).

\begin{enumerate}
\item[(i)] \textbf{Incentive to overreport (smooth case).} Suppose $q \in C^2$ in a neighbourhood of $p^*$ with $q'(p^*) \neq 0$ and $q''$ bounded on that neighbourhood. For $\wA R^*/(2\wC)$ sufficiently small, the agent's optimal report satisfies:
\begin{equation}
r^* = p^* + \frac{\wA R^*}{2\wC} \cdot q'(p^*) + O\!\left(\left(\frac{\wA R^*}{2\wC}\right)^2\right).
\end{equation}
The expansion is pointwise in $p^*$ and not uniform near the report boundary: for $p^*$ close to $1$ the optimum clamps at $r^* = 1$ (cf.\ Remark~\ref{rem:saturated-regime}). For non-monotone gates the same formula holds with the sign of $q'(p^*)$; the overreport direction uses Assumption~\ref{asm:monotone}.

\item[(ii)] \textbf{Sharp threshold.} Under $q(r) = \ind[r \geq \rmin]$, the agent inflates whenever:
\begin{equation}
\wA R^* > \wC (\rmin - p^*)^2.
\end{equation}

\item[(iii)] \textbf{Detection sample complexity.} Fix any type-II error level $\beta \in (0, 1)$ and suppose the inflated report is interior, $\Delta \ll \min(r^*, 1 - r^*)$ (i.e., $r^*$ bounded away from $0$ and $1$). Then detecting an inflation of magnitude $\Delta = r^* - p^*$ via comparison of reported confidence against the empirical success rate is sufficient with $K = O(1/\Delta^2)$ observations, and $K = \Omega(1/\Delta^2)$ observations are necessary (a matching minimax lower bound from a standard two-point Le Cam construction; see~\citealp{tsybakov2009introduction} Ch.~2). The detection sample complexity in this interior regime is therefore $\Theta(1/\Delta^2)$, matching the abstract's stated bound. At the report-space boundary (reports pooled at $r^* = 1$, the saturated regime of Remark~\ref{rem:saturated-regime}), the Bernoulli likelihoods degenerate and detection improves to $\Theta(1/\Delta)$.
\end{enumerate}
\end{lemma}

\begin{proof}
\textbf{Part (i).} Fix state $s$ and action $a^*(s)$. The agent's payoff for report $r$ is:
\begin{equation}
V(r) = \wC \cdot S(r, p^*) + \wA \cdot q(r) \cdot R(s, a^*).
\end{equation}
Using the Brier score $S(r,p) = -(r-p)^2$, the first-order condition $V'(r^*) = 0$ gives:
\begin{equation}
-2\wC(r^* - p^*) + \wA R(s,a^*) \cdot q'(r^*) = 0,
\end{equation}
yielding $r^* = p^* + \frac{\wA R(s,a^*)}{2\wC} q'(r^*)$. For $\wA R^*/(2\wC)$ small the first-order approximation follows by the implicit function theorem, and the stationary point is the \emph{global} maximum: since $q \leq 1$, any report with $\wC(r - p^*)^2 > \wA R^*$ is dominated by truthful reporting, so every maximizer lies within $\sqrt{\wA R^*/\wC}$ of $p^*$; for $\wA R^*/(2\wC)$ small this ball lies inside the $C^2$ neighbourhood, where $V'' = -2\wC + \wA R^* q'' < 0$ (bounded $q''$), so $V$ is strictly concave there and the critical point is its unique maximizer. The smallness hypothesis is not cosmetic: for gates with a second, steeper rise away from $p^*$, the global optimum jumps to the distant threshold once $\wA R^*/(2\wC)$ is large.

\textbf{Part (ii).} Under the sharp threshold $q(r) = \ind[r \geq \rmin]$, the agent compares reporting truthfully ($r = p^*$, rejected, reward: $\wC C^*$) with inflating to exactly the threshold ($r = \rmin$, approved, reward: $\wC[C^* - (\rmin - p^*)^2] + \wA R^*$). Inflation dominates when $\wA R^* > \wC(\rmin - p^*)^2$.

\textbf{Part (iii).} \emph{Sufficiency.} By Hoeffding's inequality, testing the empirical success rate against the midpoint threshold $r^* - \Delta/2$ controls both errors: each error probability is at most $e^{-K\Delta^2/2}$, so $K \geq 2\ln(1/\min(\alpha_{\text{sig}}, \beta))/\Delta^2$ observations give type-I error at most $\alpha_{\text{sig}}$ and type-II error at most $\beta$; hence $K = O(1/\Delta^2)$. \emph{Necessity.} By a two-point Le Cam argument, distinguishing $P_0 = \mathrm{Bern}(r^*)$ from $P_1 = \mathrm{Bern}(r^* - \Delta)$ over $K$ i.i.d.\ samples at a fixed error level requires $K \cdot \mathrm{KL}(P_0 \Vert P_1) = \Omega(1)$ by tensorisation of KL and Pinsker's inequality, and for $r^*$ bounded away from $\{0, 1\}$ the per-sample Bernoulli KL is $\Theta(\Delta^2 / (r^*(1 - r^*))) = \Theta(\Delta^2)$ by second-order Taylor expansion, giving $K = \Omega(1/\Delta^2)$ (see~\citealp{tsybakov2009introduction}, Theorem~2.2). The interiority hypothesis is where the Taylor step is used: at $r^* = 1$, $\mathrm{KL}(\mathrm{Bern}(1) \Vert \mathrm{Bern}(1 - \Delta)) = \ln(1/(1 - \Delta)) = \Theta(\Delta)$, the lower bound weakens to $\Omega(1/\Delta)$, and detection at the pooling boundary is correspondingly easier.
\end{proof}

Equivalently: any non-constant $q$ makes the perturbed rule $V(r) = \wC S(r,p) + \wA q(r) R$ improper --- a corollary of the classical characterizations~\citep{schervish1989general, gneiting2007strictly}, stated and proved here as Proposition~\ref{prop:properness-destruction}.

\begin{corollary}[General Inflation Formula]
\label{cor:general-inflation}
Let $S$ be a strictly proper scoring rule with convex generator $G$ (so that $\E_p[S(r,\omega)] = G(r) + G'(r)(p - r)$, maximized at $r = p$), and suppose $G \in C^2$ in a neighbourhood of $p$ with $G''(p) > 0$. Strict convexity alone does not suffice: $G(p) = (p - 1/2)^4$ is strictly convex with $G''(1/2) = 0$, and the inflation at $p = 1/2$ then scales as $(\wA/\wC)^{1/3}$ rather than first order. Under the conditions of Lemma~\ref{lem:perturbation}(i) (with $S$ in place of the Brier score), the first-order inflation for such a scoring rule is:
\begin{equation}
\label{eq:general-inflation}
\Delta(p) \;=\; \frac{\wA \cdot R(s,a^*) \cdot q'(r^*)}{\ \wC \cdot G''(p)\ },
\end{equation}
where $R(s,a^*) \in [0,1]$ is the expected reward of the optimal action. Since $G''(p) > 0$ by hypothesis, the denominator is positive and no absolute value is needed. For binary tasks where reward is normalized to $R \in \{0,1\}$ and we condition on success ($R = 1$), the formula simplifies to $\Delta(p) = \wA \cdot q'(r^*) / (\wC \cdot G''(p))$; this normalization is used throughout the paper unless stated otherwise. For the Brier score, $G''(p) = 2$ (constant), recovering $\Delta = \wA q'/(2\wC)$. For the logarithmic score, $G''(p) = 1/(p(1-p))$, giving type-dependent inflation $\Delta(p) = \wA \cdot q'(r^*) \cdot p(1-p) / \wC$.
\end{corollary}

\begin{proof}
The agent's expected payoff under generator $G$ is $\E_p[V(r)] = \wC[G(r) + G'(r)(p - r)] + \wA \cdot q(r) \cdot R(s,a^*)$. The first-order condition is $\wC \cdot G''(r^*)(p - r^*) + \wA R(s,a^*) \cdot q'(r^*) = 0$, giving $r^* - p = \wA R(s,a^*) \cdot q'(r^*) / (\wC \cdot G''(r^*))$. Since $r^* - p = O(\wA/\wC)$ and $G''$ is continuous, $G''(r^*) = G''(p) + o(1)$ as $\wA/\wC \to 0$ (with the sharper $O(\wA/\wC)$ increment under $G \in C^3$); the stated formula with $G''(p)$ in the denominator holds to first order in $\wA/\wC$.
\end{proof}

\begin{remark}
The type-dependence of inflation under non-quadratic generators has welfare implications: under the log score, agents with extreme beliefs ($p$ near $0$ or $1$) inflate less than agents with intermediate beliefs. This interacts with the principal's screening problem, whose first-best characterization under general generators remains open; the companion preprint establishes a one-sided welfare bound as a first step (see Remark~\ref{rem:saturated-regime}).
\end{remark}

\subsection{The Trilemma Theorems}
\label{sec:theorems}

\begin{theorem}[Calibration-Autonomy Two-Way Impossibility]
\label{thm:two-way}
Let $\cM$ be a CGDP satisfying Assumptions~\ref{asm:binding} and~\ref{asm:monotone}. No helpfulness-optimal CG-policy $\pi$ (that is, no $\pi$ with $H(\pi) = H^*$, equivalently one selecting $R$-maximal actions $\pi$-a.s.) simultaneously achieves:
\begin{enumerate}
\item $C(\pi) = C^*(\pi)$ (optimal calibration given the action policy), and
\item $A(\pi) = 1$ (full autonomy),
\end{enumerate}
Quantitatively, every helpfulness-optimal policy with $C(\pi) = C^*(\pi)$ satisfies $A(\pi) \leq \sup_r q - \mu'\,(\sup_r q - 1/2)$ --- in particular $A \leq 1 - \mu'/2$ for gates attaining $1$ --- which is bounded away from the unconstrained supremum whenever $\sup_r q > 1/2$; the statement is therefore non-vacuous for every monotone gate exceeding $1/2$ somewhere, including smooth gates with $\sup_r q < 1$. (The always-approve principal $q \equiv 1$ yields $\rmin = 0$ and an empty binding set, falling outside Assumption~\ref{asm:binding}'s hypotheses rather than constituting an exception.) The restriction to helpfulness-optimal policies is essential and intended: a policy that sacrifices helpfulness can achieve both conclusions jointly by selecting safe actions (the $(C, A)$ corner of Theorem~\ref{thm:trilemma}), so the two results are complementary, not contradictory.
\end{theorem}

\begin{proof}
Suppose $\pi^*$ is helpfulness-optimal and achieves both. $A(\pi^*) = 1$ requires $q(r) = 1$ $\pi^*$-a.s. $H(\pi^*) = H^*$ requires the selected action to be $R$-maximal $\pi^*$-a.s., and on $\Cbind$ every $R$-maximal action $a$ satisfies $p(s, a) < \rmin(s)$ (Assumption~\ref{asm:binding}, which covers argmax ties). $C(\pi^*) = C^*(\pi^*)$ requires $r = p(s, a)$ $\pi^*$-a.s.\ (strict properness of $S$). Hence on $\Cbind$ (measure $\mu' > 0$) the report satisfies $r < \rmin(s)$, and by the definition of $\rmin$ alone (as the infimum of $\{r : q(r) \geq 1/2\}$, any report below it receives $q(r) < 1/2$; monotonicity is not needed for this step) such reports receive $q(r) < 1/2 < 1$. Contradiction: either $A < 1$ or $C < C^*(\pi^*)$ on a set of positive measure.
\end{proof}

\begin{theorem}[Behavioral Credibility Trilemma]
\label{thm:trilemma}
Let $\cM$ be a CGDP satisfying Assumptions~\ref{asm:binding} and~\ref{asm:safe-rich}. No CG-policy $\pi$ simultaneously achieves:
\begin{enumerate}
\item $H(\pi) = H^*$ (maximum helpfulness),
\item $C(\pi) = C^*(\pi)$ (optimal calibration),
\item $A(\pi) = 1$ (full autonomy).
\end{enumerate}
Moreover, under the step gate $q(r) = \ind\{r \geq \rmin\}$ any two of the three are jointly achievable at the cost of the third. For smooth gates with $\sup_r q(r) < 1$, the corners split: the $(H,A)$ clause holds with the attainable supremum $\sup_\pi A(\pi) = \sup_r q(r)$ in place of $A = 1$ (achieved by reporting $r \in \arg\max_r q$), while the $(C,A)$ clause holds with the safe-action construction attaining $A = \E[q(p(s, a_{\text{safe}}(s)))]$ (a lower bound on the calibration-constrained supremum $\sup\{A(\pi) : C(\pi) = C^*(\pi)\}$), generically strictly below $\sup_r q$ because calibration pins the report to $p(s, a_{\text{safe}}(s))$. The $(C, A)$ corner is achieved only by policies with $H(\pi) < H^*$, outside the quantifier of Theorem~\ref{thm:two-way}.
\end{theorem}

\begin{proof}
\textbf{Impossibility.} Suppose $\pi^*$ achieves all three. $H = H^*$ requires the selected action to be $R$-maximal $\pi^*$-a.s.; on $\Cbind$ every $R$-maximal action $a$ satisfies $p(s, a) < \rmin$ (Assumption~\ref{asm:binding}, which covers argmax ties). $C = C^*(\pi^*)$ pins $r = p(s, a)$ for the selected action (strict properness). On $\Cbind$ the report therefore satisfies $r < \rmin$, so $q(r) < 1/2$ by the definition of $\rmin$, contradicting $A = 1$. Only Assumption~\ref{asm:binding} and strict properness of $S$ are used in this direction.

\textbf{Pairwise achievability.} Shown for the canonical step gate $q(r) = \ind\{r \geq \rmin\}$, which attains $q = 1$. For smooth gates the constructions must be adapted per the statement's proviso: the $(H,A)$ construction reports $r \in \arg\max_r q$ (calibration is being sacrificed, so the report is free), achieving $A = \sup_r q(r)$; the $(C,A)$ construction is unchanged but attains $A = \E[q(p(s, a_{\text{safe}}(s)))] < \sup_r q$ in general, since $C = C^*$ pins the report to $p(s, a_{\text{safe}}(s))$, which need not maximize $q$.

\emph{$(H, C)$ without $A$:} Select $a^*(s)$, report $r = p(s,a^*(s))$. On binding states, the gate rejects ($q(r) = 0$), so $A \leq 1 - \mu' < 1$. (Here $H = H^*$ counts the \emph{selected} action per Definition~\ref{def:helpfulness}; realized task reward on rejected states depends on the fate of the blocked mass --- the dimension separating the two resolutions of Section~\ref{sec:resolution}.)

\emph{$(H, A)$ without $C$:} Select $a^*(s)$, report $r = \max(p(s,a^*(s)),\, \rmin(s))$, which lies in $[0,1]$ since $\rmin \leq 1$ for the step gate. The gate approves everywhere ($A = 1$), but reports are inflated on $\Cbind$ ($C < C^*$).

\emph{$(C, A)$ without $H$:} Select $a_{\text{safe}}(s)$ with $p(s, a_{\text{safe}}(s)) \geq \rmin(s)$, which exists on every context by Assumption~\ref{asm:safe-rich} (on contexts with $p(s, a^*(s)) \geq \rmin$ the optimal action itself serves), and report truthfully $r = p(s, a_{\text{safe}}(s))$. The gate approves every report ($A = 1$ exactly), reports are calibrated ($C = C^*$), but on $\Cbind$ no $R$-maximal action clears the gate (Assumption~\ref{asm:binding}), so $a_{\text{safe}}$ is not $R$-maximal there and $H < H^*$.
\end{proof}

\paragraph{Achievable region and operating point.}
The three pairwise constructions correspond to the behavioral patterns of Section~\ref{sec:intro}. Weighted-sum optima of the principal's welfare objective $\wH H + \wC C + \wA A$ (with $\wH, \wC, \wA > 0$) trace a Pareto frontier of the achievable region $\cF = \{(H(\pi), C(\pi), A(\pi)) : \pi \in \Pi\}$, convex by randomisation over the CG-policies $\Pi$ (Definition~\ref{def:cgpolicy}); the weight vector selects the operating point, and the impossibility (Theorem~\ref{thm:trilemma}) is the statement that this frontier does not reach the joint corner $(H^*, C^*, 1)$ when the binding set has positive measure.

\subsection{Optimal Oversight Is Non-Affine}
\label{sec:non-affine}

The theorems above show that non-affine approval rules destroy strict properness; might a sophisticated principal simply choose an affine rule? The following theorem resolves this.

\paragraph{The principal's Stackelberg program.}
The CGDP induces a Stackelberg game (in the sense of the equilibrium-concept paragraph in Section~\ref{sec:intro}) in which the principal commits to an approval function $q: [0,1] \to [0,1]$, and the agent best-responds by selecting, for each type $p$, the report $r^*_q(p)$ that maximises~\eqref{eq:payoff} while playing the $R$-maximal action $a^*(s)$ --- a \emph{report-only} best response. The action-pinning is a substantive hypothesis rather than a consequence of~\eqref{eq:payoff}, which carries no helpfulness term: where a gate-clearing safe action exists (Assumption~\ref{asm:safe-rich}), substituting it can strictly dominate report inflation for some types, a joint action--report deviation outside this theorem's scope. (The first-order characterisation of the report response is Lemma~\ref{lem:perturbation}; under strict properness of $S$, Assumption~\ref{asm:monotone}, and $F$ atomless (atomless at $p_{\min}$ suffices for the step gate), the best response is unique up to a measure-zero indifference set.) Writing $\Lambda(p) := R^* \cdot \ind\{p \geq p_{\min}\} - c \cdot \ind\{p < p_{\min}\}$ for the principal's per-type net benefit from approval (with $c > 0$ the cost of approving an unprofitable type, so $\Lambda(p) > 0$ for $p \geq p_{\min}$, $\Lambda(p) < 0$ for $p < p_{\min}$, and the sign of $\Lambda$ changes at $p_{\min} \in (0, 1)$), the principal's expected utility is
\begin{equation}\label{eq:principal-program}
U_P(q) \;=\; \E_{p \sim F}\bigl[\, q\bigl(r^*_q(p)\bigr) \cdot \Lambda(p)\, \bigr].
\end{equation}
Here $F$ denotes the distribution of the agent's type $p = p(s, a^*(s))$ induced by the context distribution. The principal's Stackelberg program is to maximise $U_P(q)$ over the class of monotone non-decreasing, upper semicontinuous $q$ (Assumption~\ref{asm:monotone}). Theorem~\ref{thm:non-affine} below characterises the optimum.

\begin{theorem}[Optimal Oversight Non-Affinity]
\label{thm:non-affine}
Consider the principal's mechanism design problem: commit to a monotone non-decreasing, upper semicontinuous $q: [0,1] \to [0,1]$ (Assumption~\ref{asm:monotone}; upper semicontinuity ensures the agent's argmax is attained) to maximize expected utility $U_P(q)$, where $r^*_q(p)$ is the agent's optimal report under $q$ with the action pinned at $a^*(s)$ (the report-only best response of the program paragraph above), the principal's net benefit from approval changes sign at $p_{\min} \in (0,1)$, and the type distribution places positive mass on both sides. Suppose calibration is scored by the Brier score $S(r, p) = -(r - p)^2$, and that the task reward of the optimal action is a single constant across the binding set, $R(s, a^*(s)) \equiv R^*$ (homogeneous net stakes).

\begin{enumerate}
\item[(i)] No affine $q$ attains the first-best value $\E_{p \sim F}[\max(\Lambda(p), 0)]$; consequently, in the unsaturated regime of part (ii), where that value is attained, no affine $q$ is optimal.
\item[(ii)] When $\wA R^*/\wC \leq (1 - p_{\min})^2$, the step function $q^*(r) = \ind\{r \geq r_0\}$ with $r_0 = p_{\min} + \sqrt{\wA R^*/\wC} \leq 1$ achieves first-best screening. When $\wA R^*/\wC > (1 - p_{\min})^2$, no threshold on $[0,1]$ implements first-best: even the maximum report $r = 1$ is insufficient to separate all types, and the principal must accept pooling at $r = 1$ or employ non-threshold mechanisms. Both the threshold formula and the saturation boundary $(1 - p_{\min})^2$ are specific to the Brier score, whose inflation cost from $p$ to $r$ is the quadratic $\wC (r - p)^2$.
\item[(iii)] In the unsaturated regime $\wA R^*/\wC \leq (1 - p_{\min})^2$, an optimal gate exists (the step rule of part (ii) attains the first-best value), every optimal gate is non-affine, and the step optimum satisfies Assumption~\ref{asm:nonaffine} and falls under the sharp-threshold case of Lemma~\ref{lem:perturbation}(ii). In the saturated regime we do not establish that the supremum of $U_P$ over monotone gates is attained, and the first-best comparison of part (i) no longer yields an optimality exclusion by itself; specific affine gates (including the always-approve and never-approve corners) admit strict monotone improvements in the instances recorded in Remark~\ref{rem:saturated-regime}, but we do not claim that every optimal gate, should one exist, is non-affine.
\end{enumerate}
\end{theorem}

\begin{proof}
\emph{Part (i).} If $q(r) = a + br$, the agent's Brier FOC gives $r^*(p) = p + \wA R^* b/(2\wC) \equiv p + \delta$ wherever the shifted report is interior, so the induced screening $\tilde{q}(p) = (a + b\delta) + bp$ is affine in~$p$ there, and clipped-affine (hence still continuous) where $r^*(p)$ hits the report-space boundary. Since the principal's net benefit changes sign at $p_{\min}$, the pointwise optimal screening is the step function $\ind\{p \geq p_{\min}\}$. Any affine or clipped-affine $\tilde{q}$ either approves unprofitable types or rejects profitable ones, incurring strict loss.

\emph{Part (ii).} Under $q^*(r) = \ind\{r \geq r_0\}$ with $r_0 = p_{\min} + \sqrt{\wA R^*/\wC}$, type $p$ inflates to $r_0$ iff $\wA R^* \geq \wC(r_0 - p)^2$, i.e., $p \geq p_{\min}$; here the constancy of $R^*$ across the binding set is used, since a type-dependent reward would give type-dependent inflation cutoffs. The induced screening is exactly $\ind\{p \geq p_{\min}\}$, the first-best. (Best-response uniqueness at $r_0$, strict preference for truthful reporting for types below $p_{\min}$ and above $r_0$, and the tie at $p = p_{\min}$ are verified in Appendix~\ref{app:boundary}.)

\emph{Part (iii).} Since $q \in [0,1]$, the integrand of~\eqref{eq:principal-program} satisfies $q(r^*_q(p)) \cdot \Lambda(p) \leq \max(\Lambda(p), 0)$ pointwise, so $U_P(q) \leq \E_{p \sim F}[\max(\Lambda(p), 0)]$ for every gate. In the unsaturated regime the step rule of part (ii) induces exactly the screening $\ind\{p \geq p_{\min}\}$ and attains this bound, so the optimum exists, and by part (i) no affine gate attains it: every optimal gate is non-affine. The step optimum is non-affine on every interval crossing $r_0$, hence satisfies Assumption~\ref{asm:nonaffine} under its existential quantification, and Lemma~\ref{lem:perturbation}(ii) applies to it. In the saturated regime the bound is not attained by the step rule, so the first-best comparison of part (i) excludes nothing by itself, and attainment of the supremum over the monotone class is not established here; instance-level strict improvements over affine corner gates in that regime are recorded in Remark~\ref{rem:saturated-regime}.
\end{proof}

The homogeneity of $R^*$ in part (ii) is not cosmetic: when $R(s, a^*(s))$ varies across the binding set, types with different stakes face different inflation cutoffs $r_0 - \sqrt{\wA R(s, a^*(s))/\wC}$, and there are instances in which no single threshold $r_0 \in [0, 1]$ implements first-best screening.

\begin{remark}[Saturated regime]
\label{rem:saturated-regime}
In the saturated case of Theorem~\ref{thm:non-affine}(ii), a positive-measure set of types $p < p_{\min}$ inflates to $r = 1$ and clears the gate even at the boundary threshold $r_0 = 1$. Within the threshold class, the principal's second-best is the boundary choice $r_0 = 1$ (pool all approving types at $r = 1$), with welfare strictly below first-best by the principal's cost of approving the inflating-but-unprofitable types: $L_{\textup{sat}}(F) = \int_{p < p_{\min}} \ind\{\wA R^* \geq \wC(1 - p)^2\} \cdot |\Lambda(p)|\, dF(p) > 0$ whenever $F$ places positive mass on the corresponding subset of $[0, p_{\min})$. Within non-threshold mechanisms (e.g., randomised approval $q(r) \in (0, 1)$ on an interval, or score-dependent gates of the kind analysed in the implicit-gate reading of Section~\ref{sec:mapping}), the principal can extract a strictly higher second-best via a generalised Myerson screening solution; the optimal mechanism in this saturated regime is in general not a step rule. Affine gates are not rescued in this regime either, in the instances we have examined: against the always-approve corner $q \equiv 1$, a two-tier gate $q = \gamma \cdot \ind\{r \in [t, 1)\} + \ind\{r = 1\}$ with $\wC(1 - p_2)^2 < \wA R^*(1 - \gamma) \leq \wC(1 - t)(1 + t - 2 p_1)$ yields a strict improvement on two-atom type distributions ($p_1 < p_{\min} \leq p_2$), and against $q \equiv 0$ a uniform gate $q \equiv \gamma$ recovers $U_P/\gamma \to \E[\max(\Lambda, 0)] > 0$ as $\gamma \downarrow 0$; a general optimality-exclusion for affine gates in the saturated regime is left open. Characterising it in closed form for arbitrary $F$ is a problem we leave open, and the companion mechanism-design treatment of arbitrary smooth gates under general scoring rules in~\citet{loven_endogeneity_2026} develops the relevant first-order machinery (welfare-gap functional in $1/G''$, reserve-price schedule for non-constant curvature).
\end{remark}

This is analogous to Myerson's \citeyearpar{myerson1981optimal} reserve price: the principal sets a threshold above the naive cutoff to compensate for strategic inflation.

Theorem~\ref{thm:non-affine} is stated for the Brier score: the part~(i) proof uses the Brier first-order condition, under which an affine gate induces an affine screening of the type, a step valid for generators with constant curvature $G''$ (the Brier affine class) but not for non-constant-curvature generators such as the log score, so part~(i) is not claimed for general strictly proper rules. Which rule to adopt as the calibration metric is a design question downstream of the impossibility, taken up in Proposition~\ref{prop:brier-unique} below.

\begin{proposition}[Brier as the distribution-robust calibration loss]
\label{prop:brier-unique}
Define the \emph{leading-order welfare loss} of a strictly proper scoring rule with generator $G$ as the leading-order term, in the autonomy-to-calibration weight ratio $\wA/\wC$, of the gap between the principal's first-best welfare and the welfare achieved under the best smooth single-threshold gate (the gap is minimized over such gates, matching the companion's object; at any single fixed gate the gap carries a $G$-independent zeroth-order term and the cross-generator comparison below would be void). For generators $G \in C^3([0,1])$ whose curvature $G''$ is bounded above and bounded away from zero on $[0,1]$, and type distributions $F$ whose density is bounded below on $\Cbind$, the companion work~\citep{loven_endogeneity_2026} develops a one-sided \emph{lower} bound on this loss scaling as $\mathrm{Var}_F(1/G'' \mid \Cbind)$ (the matching upper bound is open there), and separately shows that when the curvature is constant on the binding set the loss vanishes --- the first-best is achievable in the $C^1$-gate limit --- as in the Brier affine class $G(p) = a p^2 + b p + c$ ($a > 0$); equivalently~\citep{schervish1989general}, the Brier class is the unique one whose Schervish threshold measure is uniform (Lebesgue) on $(0,1)$. Within this regular class, the Brier score is therefore the unique rule (up to affine equivalence of its restriction to the binding and report region) whose leading-order welfare loss vanishes for every such type distribution $F$; a generator quadratic on that region but non-quadratic elsewhere shares the property, hence the restriction in the uniqueness clause. The log score lies outside the regular class on $[0,1]$ (its curvature $G''(p) = 1/(p(1-p))$ is unbounded at the endpoints) and is covered only when $\Cbind$ is bounded away from $\{0, 1\}$. This distribution-robustness property is what motivates using Brier as this paper's calibration loss throughout (in the experiments and in the RLHF--Brier mapping); the central impossibility (Theorem~\ref{thm:trilemma}) holds for all strictly proper rules and does not depend on this refinement.
\end{proposition}

\section{Resolution Pathways}
\label{sec:resolution}

The trilemma is not a dead end. We identify two constructive \emph{resolution} pathways (commitment via feasibility maps, role separation via a critic) that both attain the $(H, C)$-corner triple exactly ($H = H^*$, $C = C^*$, autonomy reduced by the blocked mass), while differing in what happens to the blocked tasks --- executed by a trusted delegate (commitment) versus vetoed and dropped (separation) --- and in whether truthful reporting is enforced or incentive-native (Table~\ref{tab:resolutions}). These parallel role-separation mechanisms in regulatory economics~\citep{tirole1986hierarchies, laffont1993theory} and the scalable-oversight agenda in AI alignment~\citep{leike2018scalable, saunders2022self, bowman2022measuring}.

\subsection{Commitment Devices: Feasibility Maps}

\paragraph{Institutional framing of delegation.}
The CGDP's information-asymmetry primitive (Assumption~\ref{asm:private}) says the principal does not observe $p(s, a)$. This is not contradicted by the commitment resolution: \emph{delegation passes the task to a competent external decision-maker} (a human-in-the-loop, an oracle, or a higher-tier model) that the agent and principal jointly trust to execute the task-optimal action. The commitment policy is therefore a \emph{system-level} claim about an agent--principal--delegate institutional triple, not a single-actor claim that the principal can independently produce $a^*(s)$. A parallel caveat applies to enforcement; see Assumption~\ref{asm:commitment-enforce} and its discussion below.

This framing aligns with the selective-prediction literature~\citep{chow1957optimum, elyaniv2010foundations, madras2018predict}, where abstention routes hard instances to a more capable external decider, and with the scalable-oversight programme~\citep{leike2018scalable, bowman2022measuring}, whose principal is a tiered system of greater composite competence. We formalize it with the auxiliary assumption:

\begin{assumption}[Competent Delegation]
\label{asm:competent-delegation}
On states $s$ where the agent delegates, the institutional principal executes the task-optimal action $a^*(s) = \arg\max_a R(s, a)$ at the same expected reward as autonomous execution by an oracle agent.
\end{assumption}

Assumption~\ref{asm:competent-delegation} is the counterpart of the agent's commitment to truthful reporting: together they characterise a competence-tiered pair in which the agent handles routine tasks and the delegate handles binding ones. Either can fail empirically, but both are testable and conceptually distinct from the trilemma's incentive-theoretic core, which holds regardless.

Delegation is not an action of the CGDP itself (Definition~\ref{def:cgmdp} gives the agent only $\cA$ and a report), so we first extend the model minimally:

\begin{definition}[Delegation Extension]
\label{def:delegation-ext}
Given a CGDP $\cM$, the \emph{delegation extension} $\cM^+$ augments the agent's action set with a delegation action $a_{\textup{del}} \notin \cA$: selecting $a_{\textup{del}}$ routes the task to the delegate, who executes the task-optimal action (Assumption~\ref{asm:competent-delegation}). Write $\mathrm{exec}(s, a) = a$ for $a \in \cA$ and $\mathrm{exec}(s, a_{\textup{del}}) = a^*(s)$ for the executed action. The objectives of Section~\ref{sec:objectives} extend to $\cM^+$ by evaluating the executed action: $H^+(\pi) = \E_s[\E_{(a,r) \sim \pi(s)}[R(s, \mathrm{exec}(s,a))]]$ (a delegated state scores the delegate's action), $C^+(\pi) = \E_s[\E_{(a,r) \sim \pi(s)}[S(r, p(s, \mathrm{exec}(s,a)))]]$ (with $C^*(\pi)$ extending to the maximum of $C^+$ given the action marginal, attained at $r = p(s, \mathrm{exec}(s,a))$ by strict properness), and $A^+(\pi) = \E_s[\E_{(a,r) \sim \pi(s)}[\ind\{a \neq a_{\textup{del}}\} \cdot q(r)]]$ (a delegated task is not submitted to the gate and contributes zero autonomy). On policies that never delegate, $(H^+, C^+, A^+)$ coincide with $(H, C, A)$.
\end{definition}

\begin{assumption}[Enforced Commitment]
\label{asm:commitment-enforce}
The principal fixes the gate $q$ and the commitment protocol (the feasibility map $\Phi$ and the delegation rule of Theorem~\ref{thm:commitment}) before the agent reports, and an enforcement technology binds the agent to the committed policy: any deviation from it is blocked or voided by the verifier. This is an institutional assumption, not a consequence of the payoff~\eqref{eq:payoff}: under~\eqref{eq:payoff} alone, a binding feasible state with $\wA R^* > \wC(\rmin - p^*)^2$ admits a strictly profitable report deviation (Lemma~\ref{lem:perturbation}(ii)), and detecting an inflation of magnitude $\Delta$ requires $\Omega(1/\Delta^2)$ observations (Lemma~\ref{lem:perturbation}(iii)), so commitment is assumed to be enforced rather than derived from detection.
\end{assumption}

\begin{theorem}[Commitment Resolution]
\label{thm:commitment}
Let $\cM^+$ be the delegation extension (Definition~\ref{def:delegation-ext}) of a CGDP $\cM$ satisfying Assumptions~\ref{asm:private} and~\ref{asm:binding}, equipped with the canonical step gate $q(r) = \ind\{r \geq \rmin\}$. Let $\Phi \subseteq \cC \times \cA$ be a feasibility map (an arbitrary measurable subset of context--action pairs fixed in advance), let $\cC_{\textup{infeasible}} = \{s \in \cC : (s, a^*(s)) \notin \Phi\}$ be the set of contexts where the optimal action is outside the feasibility map, and let $\overline{\cC}_{\textup{bind}} = \{s \in \cC : p(s, a^*(s)) < \rmin(s)\} \supseteq \Cbind$ denote the full binding set. Under Assumptions~\ref{asm:competent-delegation} and~\ref{asm:commitment-enforce}, the committed policy $\pi_\Phi$ that: selects $a^*(s)$ and acts autonomously when $(s, a^*(s)) \in \Phi$ and $p(s,a^*(s)) \geq \rmin(s)$; selects $a_{\textup{del}}$ otherwise; and always reports $r = p(s, a^*(s))$, achieves:
$H^+(\pi_\Phi) = H^*$, $C^+(\pi_\Phi) = C^*(\pi_\Phi)$, and $A^+(\pi_\Phi) = 1 - \Pr(\overline{\cC}_{\textup{bind}} \cup \cC_{\textup{infeasible}}) \leq 1 - \mu'$.
\end{theorem}

\begin{proof}
On every state the executed action is $a^*(s)$: autonomously by construction of $\pi_\Phi$ outside $\overline{\cC}_{\textup{bind}} \cup \cC_{\textup{infeasible}}$, and through the delegate on that set ($\mathrm{exec}(s, a_{\textup{del}}) = a^*(s)$ by Definition~\ref{def:delegation-ext} and Assumption~\ref{asm:competent-delegation}). Hence $H^+(\pi_\Phi) = \E_s[R(s, a^*(s))] = H^*$. The report is $r = p(s, a^*(s)) = p(s, \mathrm{exec}(s,a))$ on every state, so strict properness of $S$ gives $C^+(\pi_\Phi) = C^*(\pi_\Phi)$. For autonomy, delegated states contribute zero to $A^+$ by Definition~\ref{def:delegation-ext}; on autonomous states $p(s, a^*(s)) \geq \rmin(s)$, so the truthful report clears the step gate with $q(r) = 1$. Hence $A^+(\pi_\Phi) = \Pr(s \notin \overline{\cC}_{\textup{bind}} \cup \cC_{\textup{infeasible}}) = 1 - \Pr(\overline{\cC}_{\textup{bind}} \cup \cC_{\textup{infeasible}})$, and the bound $\leq 1 - \mu'$ follows from $\Cbind \subseteq \overline{\cC}_{\textup{bind}}$ (Assumption~\ref{asm:binding}). The agent follows $\pi_\Phi$ by Assumption~\ref{asm:commitment-enforce}; no incentive-compatibility claim is made, and none follows from the payoff~\eqref{eq:payoff} alone.
\end{proof}

\subsection{Role Separation: Critic Architecture}

The separation here is between \emph{roles} within the agent-side architecture: an acting role that selects actions and a reporting role (a critic) that produces the gate input, each trained only on its own term of the composite objective. The critic must share the actor's private information, stated explicitly as this pathway's counterpart of Assumption~\ref{asm:competent-delegation}:

\begin{assumption}[Informed Critic]
\label{asm:informed-critic}
The critic observes $p(s, a)$ for the actor's selected action (or a sufficient statistic), extending the agent-side information access of Assumption~\ref{asm:private} to the critic; the principal still does not observe $p(s, a)$. The critic is trained to maximize $-\wC \cdot \E[(r' - p)^2]$, with no autonomy term: the reporting role has $\wA = 0$.
\end{assumption}

\begin{proposition}[Role Separation]
\label{prop:separation}
Assume the canonical step gate $q(r) = \ind\{r \geq \rmin\}$, Assumptions~\ref{asm:private}, \ref{asm:binding}, and~\ref{asm:informed-critic}, and let $\bar\mu' = \Pr(s \in \overline{\cC}_{\textup{bind}})$ be the mass of the full binding set of Theorem~\ref{thm:commitment}. Suppose the actor selects $a^*(s)$ in every state, the critic's report $r'$ substitutes for the actor's $r$ at the gate, and vetoed states are not executed (no delegation fallback). Then the actor--critic pair achieves $C = C^*$, $A = 1 - \bar\mu'$, and $H = H^*$ under Definition~\ref{def:helpfulness}, exactly as on the $(H, C)$ corner of Theorem~\ref{thm:trilemma}; the \emph{realized} task reward, which is zero on vetoed states, falls short of $H^*$ by at most $\bar\mu' \cdot R_{\max}$, where $R_{\max} = \sup_{s, a} R(s, a) \leq 1$.
\end{proposition}

\begin{proof}
The critic's effective payoff $V_{\text{critic}}(r') = -\wC(r' - p)^2$ is strictly proper and the critic observes $p(s, a^*(s))$ (Assumption~\ref{asm:informed-critic}), so the critic's optimal report is $r' = p(s, a^*(s))$ and $C = C^*$. Under the step gate the truthful report clears the gate exactly when $p(s, a^*(s)) \geq \rmin(s)$, so the veto set is $\overline{\cC}_{\textup{bind}}$ and $A = \E[q(r')] = 1 - \bar\mu'$. The actor selects $R$-maximal actions in every state, so $H = H^*$ by Definition~\ref{def:helpfulness} (which evaluates the selected action and is approval-independent), the same accounting as the $(H, C)$ corner of Theorem~\ref{thm:trilemma}; what is sacrificed is realized task reward, which is zero on the vetoed mass $\bar\mu'$, hence at least $H^* - \bar\mu' \cdot R_{\max}$. What separation adds relative to that corner is incentive robustness: the gate input is produced by an entity whose objective carries no $\wA$ term, so truthful reporting is the report-setter's optimum rather than a commitment requiring enforcement (Assumption~\ref{asm:commitment-enforce}), and the inflation pressure of Lemma~\ref{lem:perturbation} does not arise. The oracle quality is load-bearing at first order for safety but only at second order for calibration: a critic observing $p$ with noise of scale $\sigma$ misclassifies an $O(\sigma)$ mass of binding states into execution --- of order exactly $\Theta(\sigma)$ when the type density is bounded above and below near $\rmin$ --- while $C$ degrades only by $O(\sigma^2)$; robustness of the veto set to critic error is therefore the binding design requirement.
\end{proof}

\begin{table*}[htbp]
\centering
\small
\caption{Two constructive resolution pathways.}
\label{tab:resolutions}
\footnotesize
\begin{tabularx}{\textwidth}{l Y Y Y Y}
\toprule
Resolution & Achieves & Fate of blocked mass; robustness & Condition & Analogue \\
\midrule
Commitment & $(H,C)$-corner triple: $H = H^*$, $C = C^*$, $A \leq 1 - \mu'$ & Delegated and \emph{executed} by the delegate (realized reward preserved); truthfulness \emph{enforced} & Verifiable map $\Phi$, enforced commitment, competent delegate & Commitment device~\citep{schelling1960strategy} \\
Separation & $(H,C)$-corner triple: $H = H^*$, $C = C^*$, $A = 1 - \bar\mu'$ & Vetoed and \emph{dropped} (realized loss $\leq \bar\mu' R_{\max}$); truthfulness \emph{incentive-native} & Step gate, informed critic & Auditor separation \\
\bottomrule
\end{tabularx}
\end{table*}

Operationalizing each pathway (training procedures, architectural blueprints, convergence guarantees) is an important direction for future work.

\section{Where Do Existing Methods Sit?}
\label{sec:mapping}

\textbf{Scope note.} The CGDP formalization and Theorems~\ref{thm:two-way}--\ref{thm:non-affine} are deployment-time results: they characterize what a trained agent can achieve under oversight, however it was trained. The method mapping below extends them to training-time claims under the additional assumption that training converges to an approximate payoff maximizer --- empirically supported~\citep{schulman2017proximal, rafailov2023direct} but not formally proved for all training regimes --- so the claims in this section are structural arguments about the training objective's incentive landscape, not convergence theorems.

The CGDP framework provides a unified lens for classifying RL and alignment methods by their composite payoff structure. We express each method's training objective in the template $V(r) = \wC \cdot S(r, p) + \wA \cdot h(r)$ and determine whether the trilemma applies (Table~\ref{tab:methods}).

\subsection{The Goodhart-to-Trilemma Transition}
\label{sec:goodhart-transition}

Current RLHF/DPO methods avoid the trilemma by not attempting calibration: their overconfidence is a Goodhart effect (the reward model is a proxy for human preferences, and optimizing the proxy yields confident-sounding outputs because raters prefer them), a training artifact rather than a structural impossibility. The trilemma becomes relevant at the moment a proper scoring rule enters the training (or evaluation) loop:

\textbf{Regime~1 (Goodhart, no scoring rule):} The agent's payoff is $V = h(r)$ with $h$ non-affine and no proper-scoring component; overconfidence is unbounded (limited only by the reward model's saturation), and the resolution is a better reward model and calibration data. \textbf{Regime~2 (Trilemma, with scoring rule):} The agent's payoff is $V = \wC \cdot S(r,p) + \wA \cdot h(r)$ with $S$ strictly proper and $h$ non-affine; overconfidence is bounded (Lemma~\ref{lem:perturbation}) but structurally unavoidable, and resolution requires architectural changes (commitment, separation), not better training data.

The field is moving from Regime~1 to Regime~2: the confidence-reporting deployments of \S\ref{sec:intro} report confidence and are evaluated against ground truth, so they are CGDPs and the trilemma applies.

\subsection{Methods Without Calibration: RL, RLHF, DPO}
\label{sec:no-cal}

Plain RL ($V = R(s,a)$, with $\wC = 0$ and $q \equiv 1$ in the CGDP decomposition) escapes the trilemma trivially: with no confidence report there is nothing to inflate; the remaining methods in this class carry a non-affine autonomy incentive but no proper-scoring counter-term.

\begin{proposition}[Calibration-free methods occupy the H+A edge]
\label{prop:no-calibration-methods}
Let a training method optimize an objective $V(r) = h(r)$ in which the report $r$ enters only through a preference-derived or constraint term $h$, with no strictly-proper-scoring component ($S \equiv 0$). Then no term in the objective penalizes report--truth divergence: the calibration anchor that bounds inflation in Lemma~\ref{lem:perturbation} is absent. If in addition the reward channel favours confident-sounding reports, i.e.\ $h$ is non-decreasing in $r$ with a strict increase somewhere on $[p^*, 1]$ (the direction of Assumption~\ref{asm:cov-positive}, empirically documented for preference-trained reward models;~\citealp{sharma2024towards}), then on binding states the optimum satisfies $r^* > p^*$, with the inflation unconstrained by any calibration counter-term. (The objective itself contains no gate; ``binding'' is relative to the deployment gate the trained system faces, cf.\ the implicit gate of Proposition~\ref{prop:rlhf-brier}.) In this empirically typical case the method sits on the ``H+A edge'' of the trilemma, used informally for operating points of the achievable region $\cF$ (Section~\ref{sec:trilemma}) with $H$ and $A$ near their maxima and calibration far below $C^*$ when re-scored by an external Brier evaluation (the internal $C(\pi) = \E[S]$ degenerates under $S \equiv 0$); this is the Goodhart regime of \S\ref{sec:goodhart-transition}, not the calibration-bearing trilemma regime. This covers:
\begin{enumerate}
\item \textbf{RLHF}~\citep{christiano2017deep, schulman2017proximal}: the reward model $R_\phi$ is trained on preference comparisons, not on probability--outcome pairs, so it is not a proper scoring rule; the KL penalty constrains output-space policy divergence, not report accuracy.
\item \textbf{DPO}~\citep{rafailov2023direct} and preference-optimization variants (KTO, IPO, ORPO, SLiC-HF): by the reward--policy bijection, the implicit reward $r_{\textup{DPO}}(y|x) = \beta \log(\pi_\theta(y|x)/\piref(y|x)) + c(x)$ is a log-likelihood ratio with no $S(r,\omega)$ term (note $r_{\textup{DPO}}$ is an unbounded implicit \emph{reward} over completions, not the confidence report $r \in [0,1]$); KTO~\citep{ethayarajh2024kto} differs in input format (a Kahneman--Tversky utility on unpaired desirability labels rather than preference pairs) but likewise contains no proper-scoring term.
\item \textbf{Safe RLHF}~\citep{dai2024safe} (and Constitutional AI~\citep{bai2022constitutional}): the safety constraint restricts the action distribution, leaving the report--probability relationship unconstrained, so an agent can be safe and arbitrarily miscalibrated simultaneously.
\end{enumerate}
\end{proposition}

\begin{proof}
The absence of a strictly proper term $S(r,p)$ in each objective is stated case-by-case in the proposition; the formal deltas are the DPO reward--policy bijection~\citep[Theorem~1]{rafailov2023direct}, through which IPO, ORPO, SLiC-HF, and KTO inherit the conclusion, and the Safe RLHF constraint's form as an expected-cost bound $\E[C_{\textup{cost}}(y|x)] \leq \tau$ in the convention of \citet{dai2024safe}, which restricts the action distribution independently of the $r$-vs-$p$ relationship. In each case the objective reduces to $V = h(r)$: no term opposes inflation. For the second claim, a strict increase of $h$ somewhere on $[p^*, 1]$ gives $\max_r h(r) > h(p^*)$, and monotonicity forces every maximizer above $p^*$, so $r^* > p^*$ with no calibration term to bound the displacement (Regime~1 of \S\ref{sec:goodhart-transition}). The H+A placement follows: the methods optimize task reward by construction, inflated reports clear any deployment gate whose threshold lies below the inflated optimum, and the external Brier evaluation is bounded away from $C^*$ by the inflation itself.
\end{proof}

\subsection{Methods With Calibration: RLHF+Brier, CG-RL}
\label{sec:with-cal}

\begin{proposition}[RLHF+Brier Enters the Trilemma Regime]
\label{prop:rlhf-brier}
Consider the \emph{gated-payoff} reading of the RLHF+Brier objective,
\begin{equation*}
V(r) = -\gamma_{\textup{pert}}(r - p)^2 + \alpha \cdot q_{\textup{implicit}}(r),
\end{equation*}
where $q_{\textup{implicit}}(r) \coloneqq \Pr(R_\phi(y|x) \geq \tau_{\textup{deploy}} \mid r(y) = r)$ is the implicit approval channel induced by the reward model's deployment threshold (a regular-conditional-probability version, assumed non-decreasing, the confidence-favouring direction of Assumption~\ref{asm:cov-positive}, and $C^2$ near $p^*$ with $q'_{\textup{implicit}}(p^*) > 0$). This gated reading is exactly the form the selection-payoff experiment instantiates (Equation~\ref{eq:selection-payoff}); for the dense reading $V(r) = -\gamma_{\textup{pert}}(r - p)^2 + \alpha \cdot R_\phi(y|x)$, the same first-order shift holds with $q'_{\textup{implicit}}(p^*)$ replaced by the mean-response slope $\partial_r \E[R_\phi \mid r]\big|_{r = p^*}$. The objective is the payoff~\eqref{eq:payoff} with Brier $S$, $\wC = \gamma_{\textup{pert}}$, $\wA = \alpha$, and $R^* = 1$. On binding states of the implicit gate, the optimal report satisfies, to first order in $\alpha/\gamma_{\textup{pert}}$,
\begin{equation}
r^* = p^* + \frac{\alpha}{2\gamma_{\textup{pert}}} \cdot q'_{\textup{implicit}}(p^*) + O\!\left(\left(\frac{\alpha}{2\gamma_{\textup{pert}}}\right)^{\!2}\right).
\end{equation}
The weight ratio $\alpha/\gamma_{\textup{pert}}$ moves the operating point monotonically between the pure-Brier endpoint ($r^* = p^*$, $C = C^*$, as $\alpha/\gamma_{\textup{pert}} \to 0$) and the calibration-free H+A endpoint of Proposition~\ref{prop:no-calibration-methods} (as $\gamma_{\textup{pert}}/\alpha \to 0$): continuously when the gate is smooth enough that the composite objective stays strictly concave, whereas under the sharp gate $q = \ind\{r \geq \rmin\}$ the sweep jumps from $r^* = p^*$ to threshold pooling at $r^* = \rmin$ at the onset $\alpha/\gamma_{\textup{pert}} = (\rmin - p^*)^2$ of Lemma~\ref{lem:perturbation}(ii), and intermediate reports in $(p^*, \rmin)$ are optimal for no weight ratio (the plateau-truncated frontier observed in \S\ref{sec:exp-h3-asymptotic} is the empirical signature of this saturation). No weight ratio achieves $H = H^*$, $C = C^*$, and $A = 1$ simultaneously, where $A(\pi) = \E[q_{\textup{implicit}}(r)]$ is the autonomy functional of Definition~\ref{def:autonomy} under the implicit gate.
\end{proposition}

\begin{proof}
Under the gated reading, $V(r) = \wC S(r, p^*) + \wA \cdot q_{\textup{implicit}}(r) \cdot R^*$ with $S$ the Brier score, $\wC = \gamma_{\textup{pert}}$, $\wA = \alpha$, and $R^* = 1$, which is the payoff~\eqref{eq:payoff} on a binding state. The hypotheses of Lemma~\ref{lem:perturbation}(i) ($q_{\textup{implicit}} \in C^2$ near $p^*$ with $q'_{\textup{implicit}}(p^*) \neq 0$) hold by assumption, so the displayed first-order shift, including its remainder term, follows verbatim; for the dense reading, replacing $q_{\textup{implicit}}$ by the mean response $m(r) = \E[R_\phi \mid r]$ in the first-order condition $2\gamma_{\textup{pert}}(r^* - p^*) = \alpha\, m'(r^*)$ gives the stated substitution. For the endpoints: as $\alpha/\gamma_{\textup{pert}} \to 0$, strict properness of the Brier term pins $r^* \to p^*$; as $\gamma_{\textup{pert}}/\alpha \to 0$, the calibration term vanishes and the optimum approaches the calibration-free optimum of Proposition~\ref{prop:no-calibration-methods} ($\arg\max_r q_{\textup{implicit}}(r)$ for a smooth gate; the threshold $\rmin$ under the sharp gate, by the comparison in Lemma~\ref{lem:perturbation}(ii)). Monotonicity of the sweep is monotone comparative statics: for $r' > r$, the difference $V(r') - V(r)$ is non-decreasing in $\alpha$ because $q_{\textup{implicit}}(r') \geq q_{\textup{implicit}}(r)$ (monotone gate), so the maximizer set is non-decreasing in $\alpha$ (and, symmetrically, non-increasing in $\gamma_{\textup{pert}}$) in the strong set order; no maximizer lies below $p^*$, since on $[0, p^*]$ both terms are non-decreasing and the Brier term strictly so. Continuity holds wherever $V$ is strictly concave (for which $\alpha \sup_r |q''_{\textup{implicit}}(r)| < 2\gamma_{\textup{pert}}$ suffices): the maximizer is then unique and varies continuously in the weights by the maximum theorem. Under the sharp gate, the two-point comparison of Lemma~\ref{lem:perturbation}(ii) shows the optimum switches from $p^*$ directly to $\rmin$ at $\alpha/\gamma_{\textup{pert}} = (\rmin - p^*)^2$, so reports in $(p^*, \rmin)$ are optimal for no weight ratio; steep smooth gates that violate the concavity bound behave like the sharp case, with a smaller jump at the corresponding onset. Finally, the impossibility direction of Theorem~\ref{thm:trilemma} uses only the binding-set condition (Assumption~\ref{asm:binding}, here with $\rmin$ defined by the implicit gate) and strict properness of the Brier term (see its proof), so no weight ratio attains $(H^*, C^*, 1)$.
\end{proof}

RL with explicit confidence gating is the CGDP itself: $q$ is explicit, $S$ scores calibration against ground truth, the trilemma's assumptions hold by construction wherever some tasks exceed reliable competence, and the impossibility applies directly --- the natural habitat of the trilemma.

\subsection{Best-of-N Sampling}
\label{sec:bon}

Best-of-$N$ selection against a reward model or verifier is a standard inference-time alignment technique~\citep{nakano2021webgpt, gao2023scaling}; \citet{gao2023scaling} quantify how it overoptimizes a proxy reward as $N$ grows, the Goodhart-regime phenomenon. Proposition~\ref{prop:bon} locates the same selection pressure in the trilemma regime, where the degraded quantity is calibration itself and the degradation persists even under a noiseless, rank-aligned verifier.

\begin{proposition}[Best-of-N Instantiates the Trilemma at Inference Time]
\label{prop:bon}
Fix a binding context ($p^* < \rmin$) and write the calibration gap of a completion $y$ as $g(y) = (r(y) - p^*)^2$. Let the Best-of-$N$ policy draw $N$ i.i.d.\ completions $y_1, \ldots, y_N \sim \pi_{\textup{base}}$ and select $i^* = \arg\max_{i \in [N]} R_\phi(y_i \mid x)$. Assume: (a)~the law of $R_\phi(y \mid x)$ under $\pi_{\textup{base}}$ is atomless; (b)~$\Pr_{y \sim \pi_{\textup{base}}}(r(y) \geq \rmin) > 0$; (c)~\emph{comonotone selection}: $R_\phi$ is almost surely non-decreasing in $r$ on $\{r \geq \rmin\}$, that is, for any two completions that both clear the gate, the one with the strictly higher report does not receive the strictly lower reward-model score. Define the gate-conditional calibration gaps $\delta_{\textup{base}} = \E_{y \sim \pi_{\textup{base}}}[g(y) \mid r \geq \rmin]$ and $\delta_{\textup{BoN}}(N) = \E[g(y_{i^*}) \mid r(y_{i^*}) \geq \rmin]$. Then:
\begin{enumerate}
\item[(i)] $\delta_{\textup{BoN}}(N) \geq \delta_{\textup{base}}$ for every $N \geq 1$, with strict inequality for $N \geq 2$ whenever $r$ is not almost surely constant on $\{r \geq \rmin\}$ (under hypothesis~(c), this is equivalent to strict positivity of the gate-conditional covariance of $R_\phi$ and $g$, the conditional form of Assumption~\ref{asm:cov-positive}).
\item[(ii)] If, in addition, $R_\phi$ is a strictly increasing function of $r$ on the support of $\pi_{\textup{base}}$ (a noiseless rank-aligned verifier), then $\delta_{\textup{BoN}}(N+1) \geq \delta_{\textup{BoN}}(N)$ for all $N$, with strict inequality whenever the law of $r$ restricted to $[\rmin, 1]$ is non-degenerate, and $\delta_{\textup{BoN}}(N) \to (\bar r - p^*)^2$ as $N \to \infty$, where $\bar r$ is the essential supremum of $r(y)$ under $\pi_{\textup{base}}$. For verifiers satisfying~(c) but not the strictly-increasing condition, part~(i) still holds at every $N$, but the sequence $\delta_{\textup{BoN}}(N)$ need not be monotone in $N$ and need not converge to $(\bar r - p^*)^2$.
\item[(iii)] The Best-of-$N$ policy is itself a CG-policy of the underlying CGDP, so the impossibility direction of Theorem~\ref{thm:trilemma} (which uses only Assumption~\ref{asm:binding} and strict properness of $S$) applies to it unchanged: no selection size $N$ and no selection criterion yields a Best-of-$N$ policy attaining the exact joint corner $(H^*, C^*(\pi), 1)$ on binding states: inference-time selection relocates the operating point within the achievable region, not beyond the Pareto frontier.
\end{enumerate}
\end{proposition}

\begin{proof}
All expectations are over $y \sim \pi_{\textup{base}}$ at the fixed binding context. By hypothesis~(a) the $\arg\max$ in the selection step is almost surely unique, and the law of the selected completion $y_{i^*}$ is absolutely continuous with respect to the law of $\pi_{\textup{base}}$ with order-statistic density
\[
w_N(y) \;=\; N \cdot \Pr\!\bigl(R_\phi(y' \mid x) \leq R_\phi(y \mid x)\bigr)^{N-1},
\]
which is non-decreasing in $R_\phi(y \mid x)$ by construction and satisfies $\E[w_N] = 1$ under the unconditional law. Conditioning the \emph{selected} completion on clearing the gate therefore gives the normalized identity
\[
\delta_{\textup{BoN}}(N) \;=\; \frac{\E\bigl[w_N(y)\, g(y)\, \ind\{r \geq \rmin\}\bigr]}{\E\bigl[w_N(y)\, \ind\{r \geq \rmin\}\bigr]},
\]
with positive denominator by hypothesis~(b). The normalization matters: $\E[w_N \mid r \geq \rmin] \neq 1$ in general, so the unnormalized expression $\E[w_N\, g \mid r \geq \rmin]$ is not the selected-conditional gap and is not the quantity bounded below.

\textbf{Part (i).} Write $\E'[\cdot] = \E[\cdot \mid r \geq \rmin]$, so that $\delta_{\textup{BoN}}(N) = \E'[w_N g]/\E'[w_N]$ and $\delta_{\textup{base}} = \E'[g]$. On $\{r \geq \rmin\}$ the gap $g(y) = (r(y) - p^*)^2$ is a strictly increasing function of $r$ (since $\rmin > p^*$), and $w_N$ is a non-decreasing function of $R_\phi$; by hypothesis~(c) the pair $(r, R_\phi)$ is almost surely comonotone there, hence so is the pair $(g, w_N)$. Hoeffding's covariance identity for comonotone pairs~\citep{hoeffding1940masstabinvariante, embrechts2002quantile} (equivalently, the Harris--Proschan covariance inequality for two functionals monotone in a common ordering,~\citealp{harris1960lower, proschan1977monotone}) gives $\mathrm{Cov}'(w_N, g) \geq 0$, so
\[
\delta_{\textup{BoN}}(N) \;=\; \frac{\E'[w_N]\, \E'[g] + \mathrm{Cov}'(w_N, g)}{\E'[w_N]} \;=\; \delta_{\textup{base}} + \frac{\mathrm{Cov}'(w_N, g)}{\E'[w_N]} \;\geq\; \delta_{\textup{base}}.
\]
The covariance of a comonotone pair vanishes only when one coordinate is almost surely constant; hypothesis~(a) rules out $w_N$ being almost surely constant on $\{r \geq \rmin\}$ for $N \geq 2$ (a constant $R_\phi$ on a positive-probability event would be an atom), so strict inequality holds exactly when $r$ (hence $g$) is not almost surely constant there. Hypothesis~(c) is doing the work in this step and is not implied by Assumption~\ref{asm:cov-positive}: a positive covariance is compatible with a non-monotone dependence of $R_\phi$ on $r$, under which $\delta_{\textup{BoN}}(N)$ can fall strictly below $\delta_{\textup{base}}$ at large $N$ (for instance, via a rarely sampled top-scoring completion mode with a small gap).

\textbf{Part (ii).} When $R_\phi$ is a strictly increasing function of $r$ on the whole support, the selected completion attains the maximal report, $r(y_{i^*}) = M_N \coloneqq \max_{i \in [N]} r(y_i)$ almost surely, and the selected completion clears the gate exactly on $\{M_N \geq \rmin\}$. The conditional law of $M_N$ given $M_N \geq \rmin$ is stochastically non-decreasing in $N$: for $x \geq \rmin$, $\Pr(M_N > x \mid M_N \geq \rmin) = (1 - F_r(x)^N)/(1 - F_r(\rmin^-)^N)$ with $F_r$ the distribution function of $r$ (the subscript distinguishes it from the type distribution $F$ of Section~\ref{sec:non-affine}), and $(1 - b^N)/(1 - a^N)$ is non-decreasing in $N$ for $0 \leq a \leq b \leq 1$, $a < 1$. Since $x \mapsto (x - p^*)^2$ is non-decreasing on $[\rmin, 1]$, it follows that $\delta_{\textup{BoN}}(N+1) \geq \delta_{\textup{BoN}}(N)$, strictly when the law of $r$ on $[\rmin, 1]$ is non-degenerate. As $N \to \infty$, $M_N \to \bar r$ almost surely, and bounded convergence gives $\delta_{\textup{BoN}}(N) \to (\bar r - p^*)^2$. Both conclusions use the strictly-increasing hypothesis essentially: under a noisy verifier satisfying only~(c), the selected completion is the $R_\phi$-maximizer rather than the report-maximizer, the two decouple, and the sequence can plateau strictly below $(\bar r - p^*)^2$ while part~(i)'s per-$N$ bound continues to hold.

\textbf{Part (iii).} Composing the base policy with Best-of-$N$ selection defines a CG-policy $\pi_{\textup{BoN}}: \cC \to \Delta(\cA \times [0,1])$ of the same CGDP. The impossibility direction of Theorem~\ref{thm:trilemma}, whose hypotheses are properties of the decision problem, not of how the policy was produced, applies verbatim to $\pi_{\textup{BoN}}$: the exact joint corner $(H^*, C^*(\pi_{\textup{BoN}}), 1)$ is infeasible. No stronger comparative claim (in particular, that inference-time selection cannot improve helpfulness while preserving calibration relative to the base policy) follows from Theorem~\ref{thm:trilemma}, and none is asserted here.
\end{proof}

\begin{table}[htbp]
\centering
\small
\caption{Existing methods mapped onto the trilemma. Edge labels are informal placements (Proposition~\ref{prop:no-calibration-methods}); the Best-of-N row's degradation reading holds under Proposition~\ref{prop:bon}'s comonotone-selection hypothesis.}
\label{tab:methods}
\setlength{\tabcolsep}{4pt}
\begin{tabular}{llccl}
\toprule
Method & $S(r,p)$ & Trilemma? & Optimizes & Position \\
\midrule
Plain RL & None & No & $H$ & $(H^*, \text{N/A (no report)}, 1)$ \\
RLHF & None & No & $H, A$ & H+A edge \\
DPO & None & No & $H, A$ & H+A edge \\
Constitutional AI & None & No & $H$, partial $C$ & Between H+A, H+C \\
Safe RLHF & None & No & $H$ (constrained) & Near H+A \\
RLHF + Brier & $-(r-p)^2$ & \textbf{Yes} & $C, A$ trade-off & C $\leftrightarrow$ H+A sweep \\
CG-RL (explicit gate) & $S(r,p)$ & \textbf{Yes} & $(H,C,A)$ Pareto & On Pareto surface \\
Best-of-N & Inherited & \textbf{Yes} ($\wA > 0$) & $H$ & Shifts base $\to$ H+A \\
\bottomrule
\end{tabular}
\end{table}

\section{Optimizer-Independence}
\label{sec:optimizer}

An optimizer ascending the composite payoff on binding states ends at an inflated report whenever it reaches a local maximizer of a smooth-gate payoff with $q'(p^*) \neq 0$; under the sharp gate, calibration itself remains a local maximizer, and only the global comparison of Lemma~\ref{lem:perturbation}(ii) forces inflation. The theorem below makes the scope precise (which optimizers, which policy families, and the sharp-gate caveat).

\begin{theorem}[Optimizer-Independence]
\label{thm:optimizer-independence}
Let $V: [0,1] \to \R$ be the composite payoff~\eqref{eq:payoff} with $S$ strictly proper, $q$ non-decreasing (Assumption~\ref{asm:monotone}) and non-affine (Assumption~\ref{asm:nonaffine}), and $\wA > 0$. Assume the inflation regime of Lemma~\ref{lem:perturbation}(ii) holds on binding states (for a smooth gate with $q'(p^*) \neq 0$ and $R^* > 0$ this is automatic; for the sharp-threshold gate it is the condition $\wA R^* > \wC(\rmin - p^*)^2$). Let $r^* \in \arg\max_r V(r)$ (nonempty for upper semicontinuous $q$, which we assume; every conclusion below holds for every maximizer) and $r_{\textup{cal}} = p$. Then $r^* \neq r_{\textup{cal}}$ on binding states, and:

\begin{enumerate}
\item[(i)] \textup{(Rational agent.)} A Bayesian expected-utility maximizer reports $r^*$, not $r_{\textup{cal}}$.

\item[(ii)] \textup{(Policy gradient.)} Let $S$ be the Brier score and let the policy be a location family $r = \mu_\theta + \varepsilon$, where the noise $\varepsilon$ has a density that is symmetric about zero, log-concave, and of full support (the Gaussian-mean family is the canonical case; $q$ and $S$ are extended to $\R$ in the natural way, with $q$ constant outside $[0,1]$). Then $\nabla_\theta \E[V(r_\theta)]$ points away from the calibrated policy at $\mu_\theta = p$.

\item[(iii)] \textup{(Evolutionary.)} Any infinite-population evolutionary dynamic with fitness-monotone selection (weights proportional to a strictly increasing, \emph{positive} transform of $V$, e.g.\ $e^{V}$; deterministic replicator / exponential-weights dynamics, without mutation or finite-population drift) and full-support initialization on $[0,1]$ concentrates mass on $r^*$, not on $r_{\textup{cal}}$.
\end{enumerate}
\end{theorem}

\begin{proof}
Part (i) is immediate from $r^* \neq p$.

\textbf{Part (ii): Policy gradient for symmetric location families under the Brier score.} Consider first the canonical Gaussian case $r \sim \mathcal{N}(\mu_\theta(s), \sigma^2)$ and write $J(\theta) = \E_{r \sim \pi_\theta}[V(r)]$ for the expected payoff. At $\mu_\theta = p$:
\begin{equation}
\frac{\partial J}{\partial \mu_\theta}(s) = \frac{1}{\sigma^2}\E_{r \sim \pi_\theta}\left[V(r) \cdot (r - p)\right] = \frac{1}{\sigma^2}\text{Cov}_{\pi_\theta}(V(r), r).
\end{equation}

Write $V(r) = -\wC(r-p)^2 + \wA \cdot q(r) \cdot R$. Since $r \sim \mathcal{N}(p, \sigma^2)$ is symmetric about $p$:
\begin{equation}
\text{Cov}(V(r), r) = -\wC \underbrace{\E[(r-p)^3]}_{= 0 \text{ (symmetry)}} + \wA R \cdot \underbrace{\text{Cov}(q(r), r)}_{> 0 \text{ (Chebyshev)}}.
\end{equation}

The second term is positive by the classical one-dimensional Chebyshev association inequality: for $f$ monotone non-decreasing and non-constant on the support of $X$, $\text{Cov}(f(X), X) > 0$ for any non-degenerate law of $X$ (immediate from Hoeffding's covariance identity~\citep{hoeffding1940masstabinvariante}; no density-shape hypothesis is needed for this one-dimensional step, and the lattice-based FKG inequality of~\citet{fortuin1971correlation} is the higher-dimensional generalisation). Strict inequality requires $q$ non-constant (not merely non-affine; non-constant is the weaker and correct condition), which full support guarantees: the support of $\pi_\theta$ meets the increase region of $q$. Since $q$ is non-decreasing (Assumption~\ref{asm:monotone}) and non-constant on the support of $\pi_\theta$:
\begin{equation}
\text{Cov}(V(r), r) = \wA R \cdot \text{Cov}(q(r), r) > 0,
\end{equation}
so $\partial J / \partial \mu_\theta > 0$, meaning the gradient pushes the mean report above $p$.

For a general location family $r = \mu_\theta + \varepsilon$ as in the statement, the score function of the location parameter is $s_\rho(r - \mu_\theta) = -\rho'(r - \mu_\theta)/\rho(r - \mu_\theta)$ for noise density $\rho$, and the policy gradient at $\mu_\theta = p$ is $\E[V(r)\, s_\rho(r - p)]$. Symmetry of $\rho$ makes $s_\rho$ odd, so the Brier term $-\wC \E[(r-p)^2 s_\rho(r-p)]$ vanishes as an even-times-odd integral; log-concavity of $\rho$ makes $s_\rho$ non-decreasing, so the gate term equals $\wA R \cdot \text{Cov}(q(r), s_\rho(r-p)) > 0$ by the same association inequality applied to the non-decreasing pair $q(\cdot)$ and $s_\rho(\cdot - p)$. The hypotheses cannot be dropped within the log-concave class: for the truncated-Gaussian family on $[0,1]$ (location $0.2$, scale $0.15$, step gate at $\rmin = 0.7$, $\wC = R^* = 1$) the truncation skew makes the Brier term positive, the calibrated policy is exactly stationary at $\wA \approx 7.61$, and the gradient points below $p$ (deflation) for smaller $\wA$ still inside the inflation regime; the mean-parameterized Beta$(2, 8)$ family is likewise exactly stationary at calibration at $\wA \approx 4.81$. Asymmetric or non-location log-concave families can therefore have a stationary or even deflationary calibrated policy, which is why part (ii) is stated for symmetric full-support location families under the Brier score.

\textbf{Part (iii): Evolutionary dynamics and ascent optimizers.} Under selection weights proportional to $e^{V}$ (or any strictly increasing transform of $V$; raw proportionality to $V$ is ill-defined where $V < 0$), the population law follows replicator dynamics $p_t \propto p_0\, e^{tV}$, which concentrates mass on the maximizers of $V$ over the initial support. With full-support initialization (or any selection-mutation dynamics whose mutation kernel maintains full support) this set is $\{r^*\}$, and under the inflation regime $r^* \neq p$: $r^* = \rmin$ for the sharp-threshold gate by Lemma~\ref{lem:perturbation}(ii), and $r^* > p$ for a smooth gate with $q'(p^*) \neq 0$ by Lemma~\ref{lem:perturbation}(i). The full-support hypothesis is necessary: an initial population supported strictly below the gate's increase region concentrates on the calibrated report instead. For generic ascent optimizers the guarantee is local rather than global. Writing $V$ with the Brier calibration term as above, for a smooth gate with $q'(p^*) \neq 0$ and $R > 0$ every local maximizer of $V$ lies strictly above $p$ (for $r \leq p$, $V'(r) = -2\wC(r - p) + \wA R\, q'(r) > 0$), so any ascent method that converges to a local maximizer of $V$ exhibits strictly positive inflation; convergence to the \emph{global} peak $r^*$ additionally requires $V$ to be unimodal on $[0,1]$, since a steep sigmoid gate makes $V$ bimodal and gradient ascent initialized at calibration then stops at the lower local maximum (which still lies strictly above $p$). Weak improvement alone ($\E_{r_{t+1}}[V] \geq \E_{r_t}[V]$) guarantees neither global nor local-maximizer convergence. Under the sharp-threshold gate $q(r) = \ind\{r \geq \rmin\}$ the payoff is piecewise quadratic with a local maximum at $r = p < \rmin$ and, under the inflation condition $\wA R^* > \wC(\rmin - p^*)^2$ of Lemma~\ref{lem:perturbation}(ii), a strictly dominant global maximum at $r^* = \rmin$; local search initialized at calibration can remain at $p$, and reaching $\rmin$ requires global exploration (population-level policy gradient with non-vanishing exploration, or simulated annealing), not mere ascent. Under small $\wA / \wC$ failing the inflation condition, the global maximum is at $r = p$ and the conclusion flips; the theorem's "moves away from calibration" statement therefore requires the inflation regime.
\end{proof}

\begin{corollary}
\label{cor:calibration-not-stationary}
Under the hypotheses of Theorem~\ref{thm:optimizer-independence}(ii) (Brier score; symmetric, log-concave, full-support location family), the calibrated policy ($\mu_\theta = p$ for all $s$) is not a stationary point of the policy gradient on binding states with $R^* > 0$ whenever $\wA > 0$, and a gradient step from calibration moves the mean report upward. This uses only the computation in the proof of part (ii), not the inflation-regime condition: the gradient at calibration equals $\wA R^* \cdot \mathrm{Cov}(q(r), s_\rho(r - p)) > 0$ for every $\wA > 0$. Outside the symmetric location class the conclusion can fail even with $\wA > 0$: at the knife-edge instances in the proof of part (ii) the calibrated policy is exactly stationary.
\end{corollary}

\begin{remark}
The Gaussian parameterization is illustrative: the proof covers every symmetric, log-concave, full-support location family (logistic noise, for instance). Finite mixtures of log-concave densities are \emph{not} in general log-concave (a standard counterexample is any bimodal Gaussian mixture), so the argument does not extend to arbitrary mixtures.
\end{remark}

\section{Experimental Protocol and Results}
\label{sec:experiments}

This section presents a hypothesis-driven experimental protocol with explicit falsification criteria and reports the results of a 540-configuration run on Qwen-2.5-7B (54{,}000 selected-task observations): all five hypotheses are confirmed (Table~\ref{tab:hypotheses}), and \S\ref{sec:exp-h3-asymptotic} adds a \emph{descriptive} analysis of the achievable-$(H, C, A)$ surface geometry. The protocol and its confirmatory falsification criteria were pre-specified in the version-controlled repository before any experimental data were generated (the reported H1 and H2 test specifications are theory-aligned revisions of the originally frozen ones; both versions are archived and disclosed in \S\ref{sec:exp-stats}), and the full run replicates on two further open-weights families with all five Holm decisions matching the primary (Appendix~\ref{app:cross-model-replication}, Table~\ref{tab:cross-model-replication}).

The protocol uses Best-of-N selection as the optimization mechanism: each completion is an (action, confidence) pair and the selection operator is a formal $\arg\max$ over the composite payoff, instantiating the CGDP's report-space structure directly and isolating the mechanism the theorems analyze (rather than conflating instruction-following with objective optimization, as prompting-based studies do); the predictions follow from the payoff geometry, not from assumptions about the agent's internal representations.

\subsection{Experimental Objective}
\label{sec:exp-objective}

The protocol targets five theoretical predictions (and an additional descriptive surface-geometry observation). \textbf{Prediction~1} (selection-induced degradation / Proposition~\ref{prop:bon}): under a composite payoff with $\wA > 0$, Best-of-N selection degrades calibration relative to the unpenalised regime ($\wA = 0$); the attribution to Proposition~\ref{prop:bon} holds under its comonotone-selection hypothesis, which the oracle selector does not literally instantiate (robustness check in \S\ref{sec:exp-controls}). \textbf{Prediction~2} (Perturbation Lemma, sharp-gate form): inflation $\Delta$ on binding states is monotone non-decreasing in $\wA/\wC$, with a step-function onset near $\wA/\wC \approx (\rmin - p^*)^2$ and saturation at $\Delta \approx \rmin - p^*$ once the gate bonus dominates (the confirmatory H2 criterion tests the monotone trend only; onset and saturation are assessed descriptively in \S\ref{sec:exp-h3-asymptotic}). \textbf{Prediction~4} (threshold pooling): under gating pressure, the selected-confidence distribution develops excess mass in $[\rmin, \rmin + \varepsilon]$ on binding tasks. \textbf{Prediction~5} (binding-state specificity): inflation concentrates on binding tasks ($p < \rmin$), not on already-above-threshold tasks. \textbf{Prediction~6} (control): with no gating bonus ($\wA = 0$), Best-of-N \emph{improves} calibration --- selection alone is not a source of miscalibration. We retain the original hypothesis labels (H1, H2, H4, H5, H6); H3 denoted the convexity prediction, now reported as the descriptive surface-geometry observation of \S\ref{sec:exp-h3-asymptotic}.

\subsection{Task Set}
\label{sec:exp-tasks}

We construct 100 tasks with verifiable ground truth, stratified by difficulty and domain: 40 arithmetic/mathematical reasoning tasks (20 easy single-step, 20 hard multi-step), 30 factual knowledge tasks (15 well-known, 15 obscure), and 30 code generation tasks (15 simple functions, 15 algorithmic). Each task requires the model to (a)~state a confidence $r \in [0,1]$ that its answer is correct, and (b)~provide the answer; ground-truth verification is task-type-specific (\S\ref{sec:exp-compute}).

\paragraph{Task provenance.} The 100 tasks are committed to the repository (\texttt{scripts/\allowbreak{}generate\_\allowbreak{}tasks.py}, lines 85--121 for the factual set) on 2026-04-01, pre-dating every experimental run, and are fixed in the archived release. The arithmetic tasks are procedurally generated from random expressions with seeded RNG; the code tasks are hand-authored Python function signatures with hand-authored test cases; the factual tasks were generated by Claude Opus 4.6 \citep{anthropic2026claudeopus} under a prompt constrained to single-canonical-answer items from standard reference knowledge, with every proposed answer verified against public reference sources (Wikipedia, standard encyclopaedias, authoritative databases) and contested or source-dependent items excluded. Generating the factual items rather than drawing them from existing benchmarks (e.g.\ TriviaQA, MMLU) avoids training-data contamination: the generator post-dates Qwen-2.5's training cutoff, so the items cannot appear in its training distribution.

\textbf{Binding states.} Tasks where the model's true competence $p$ falls below the gating threshold $\rmin$ form the binding set $\Cbind$. We estimate $\hat{p}_t$ per task from the base model's accuracy on 20 \emph{held-out} seeds $\{1000, \ldots, 1019\}$, disjoint from the experimental seeds $\{0, 42, 123, 456, 789\}$; the orthogonal split prevents circularity in the H5 binding-state-specificity test, fixing a task's binding-set membership before any selection outcome on that task enters the analysis. Binding-set sizes: 71 tasks at $\rmin = 0.5$, 72 at $\rmin = 0.7$, 74 at $\rmin = 0.9$.

\subsection{Models}
\label{sec:exp-models}

\textbf{Primary model in this study:} Qwen-2.5-7B-Instruct (\texttt{qwen2.5:7b}, Q4\_K\_M) via Ollama: local, reproducible, open-weights, and sufficient in scale to exhibit calibration behavior while remaining tractable for Best-of-N with large $N$. All five hypothesis tests (H1, H2, H4, H5, H6), the 540-config gated mechanism experiment, and every numerical result in the Tables below are on this primary model. The full gated protocol is additionally replicated on two further open-weights families, \texttt{gemma2:9b} and \texttt{yi:9b} (Appendix~\ref{app:cross-model-replication}); a descriptive ungated placement on eight open-weights instruct architectures, a competence-controlled comparability anchor rather than a hypothesis test, is reported in Appendix~\ref{app:cross-model-descriptive}.

\subsection{Experimental Protocol}
\label{sec:exp-protocol}

For each task $t$, model $m$, weight vector $\mathbf{w}$, and selection size $N$:

\begin{enumerate}
\item \textbf{Generate.} Sample $N$ independent completions from the model, each including an answer $a_i$ and a confidence report $r_i$, at temperature $\tau = 0.8$ to ensure diversity. The report $r_i$ is derived from token-level log-probabilities returned by the inference endpoint (Ollama's OpenAI-compatible \texttt{/v1/chat/completions}, \texttt{logprobs: true}): for a completion with $T$ tokens having log-probabilities $\ell_1, \ldots, \ell_T$, we set
\begin{equation}
\label{eq:logprob-confidence}
r_i \;=\; \exp\!\left(\frac{1}{T}\sum_{t=1}^{T}\ell_t\right) \;=\; \left(\prod_{t=1}^{T} p_t\right)^{\!1/T},
\end{equation}
the geometric mean of per-token probabilities, clipped to $[0.01, 1.0]$. Token-level log-probabilities correlate well with completion-level correctness at this scale~\citep{kadavath2022language, tian2023just}; an earlier verbalised-confidence protocol produced degenerate reports and was abandoned (\S\ref{sec:exp-controls}).
\item \textbf{Score.} Compute the selection payoff for each completion. Because Best-of-N operates post-hoc on completions whose correctness can be verified (by arithmetic check, fact lookup, or test execution), the ground truth $y_i \in \{0,1\}$ \emph{is} available at selection time; the oracle selection payoff is
\begin{equation}
\label{eq:selection-payoff}
V_i^{\text{sel}} = -\wC \cdot (r_i - y_i)^2 + \wA \cdot \ind\{r_i \geq \rmin\},
\end{equation}
where the first term is the Brier score against ground truth (a strictly proper scoring rule perturbation) and the second is the gating reward, matching the Perturbation Lemma template $V = \wC \cdot S(r, y) + \wA \cdot q(r)$ exactly: the experiment directly instantiates the theorem's payoff structure.

\textbf{Oracle vs expected Brier.} The selection payoff uses the realized outcome $y_i \in \{0,1\}$, whereas the Perturbation Lemma is stated in terms of the expected-Brier payoff $V(r) = -\wC (r - p)^2 + \wA q(r)$ using the true probability $p$. The two are consistent in expectation, $\E_{y \sim \mathrm{Bernoulli}(p)}\bigl[-(r-y)^2\bigr] = -(r-p)^2 - p(1-p)$, so for every fixed $p$ the two objectives have the same $\arg\max_r$ (the $-p(1-p)$ term is constant in $r$, and the gate term $\wA q(r)$ depends on $r$ but not on $y$). Realized-outcome selection therefore adds sampling variance to which completion is selected, but does not bias the Perturbation Lemma's inflation prediction. Replacing $y_i$ in \eqref{eq:selection-payoff} with a held-out, frozen per-task estimate $\hat p_t$ (expected-Brier selection) is a natural control that we did not run and note as future work. All hypothesis results (H1, H2, H4, H5, H6) reported below use the realized-outcome payoff of \eqref{eq:selection-payoff}.

\textbf{No-oracle proxy (sensitivity check).} When ground truth is unavailable, a natural proxy replaces the Brier score with the agent's self-assessed expected calibration: $V_i^{\text{proxy}} = -\wC \cdot r_i(1 - r_i) + \wA \cdot \ind\{r_i \geq \rmin\}$. Since $-r(1-r)$ is Bernoulli variance, maximized at $r \in \{0, 1\}$ rather than at $r = p$, this proxy is \emph{not} a proper scoring rule and incentivizes extreme reports independent of the gating term; it is included only as a control quantifying the oracle/no-oracle gap. Materialized as a pre-specified Stage-B control over the archived pools, the proxy behaves exactly as this improper-score characterization predicts: being $y$-blind, it selects the extremeness argmax (mean selected $r = 0.958$, invariant across $\wA$ on these pools, which renders the weight-contrast tests degenerate for this selector by construction), yielding at $\rmin = 0.7$, $N = 32$ a selected-completion mean Brier of $0.603$ against the oracle's $0.274$--$0.368$ and placing only $2.5\%$ of binding-task mass in the threshold window against the oracle's $65.8\%$ --- overshooting the gate rather than pooling at it. All primary analyses use the oracle payoff~\eqref{eq:selection-payoff}.
\item \textbf{Select and record.} Choose completion $i^* = \arg\max_i V_i^{\text{sel}}$, store the selected $(r_{i^*}, a_{i^*})$, and compute all post-hoc metrics.
\end{enumerate}

\subsection{Independent and Dependent Variables}
\label{sec:exp-iv}

The independent variables are \textbf{selection pressure} $N \in \{1, 2, 4, 8, 16, 32\}$ ($N = 1$ is the base model, no selection), \textbf{autonomy pressure} $\wA/\wC \in \{0, 0.25, 0.5, 1.0, 2.0, 4.0\}$ ($\wA/\wC = 0$ is pure calibration, no gating reward), \textbf{gating threshold} $\rmin \in \{0.5, 0.7, 0.9\}$, and 5 random \textbf{seeds} per configuration. Total configurations: $6 \times 6 \times 3 \times 5 = 540$ runs, evaluating ${\approx}567$K completion-slots ($\sum_N N \times 100$ tasks across the weight, threshold, and seed grid); the slots are filled from $16{,}000$ unique pooled completions ($100$ tasks $\times$ $5$ seeds $\times$ $N_{\max} = 32$), generated once in Stage~A (${\sim}14$ hours on Qwen-2.5-7B) and reused deterministically across configurations (\S\ref{sec:exp-compute}).

\label{sec:exp-dv}
The dependent variables, computed per configuration over the 100-task set, are \textbf{calibration} $\text{BS} = \frac{1}{T}\sum_{t=1}^{T}(r_t - y_t)^2$, with $r_t$ the selected confidence and $y_t \in \{0,1\}$ correctness (lower is better; the standard binned reliability--resolution--uncertainty decomposition of \citealp{murphy1973new} is computed and archived in the repository but not analysed here); \textbf{helpfulness} $H = \frac{1}{T}\sum_{t=1}^{T} R(a_t, t)$, the fraction of tasks answered correctly; \textbf{autonomy} $A = \frac{1}{T}\sum_{t=1}^{T} \ind\{r_t \geq \rmin\}$, the fraction of tasks where the agent clears the gate; \textbf{inflation} $\Delta = \frac{1}{|\Cbind|}\sum_{t \in \Cbind}(r_t - p_t)$, mean confidence inflation on binding states, where $p_t$ is the base model's empirical accuracy on task~$t$.

\subsection{Hypotheses, Falsification Criteria, and Statistical Methodology}
\label{sec:exp-hypotheses}

Table~\ref{tab:hypotheses} lists the five hypotheses with their falsification criteria and the results from the 540-configuration run.

The Welch $t$-test and Cohen's $d$ for H5 in Table~\ref{tab:hypotheses} use per-completion comparisons between binding and non-binding tasks, which inflates the effect size relative to a per-task mixed-effects estimator because non-binding inflation is near zero by construction; the direction and significance are robust, and the per-task aggregate difference remains highly significant and recoverable from the archived per-config metrics in the repository.

\label{sec:exp-stats}
The hypothesis tests use the following specifications: paired $t$ for H1 and H6 on per-task Brier; Jonckheere--Terpstra on per-seed mean inflation for H2; two-proportion $z$ on per-completion threshold-window mass for H4; Welch $t$ on per-completion inflation for H5. For H1, H2, and H6 the aggregation to task or seed level absorbs per-task difficulty and per-seed sampling variation before the weight manipulation is tested; H4 and H5 are per-completion by design, so their statistics treat pooled completions as independent even though the $N = 1$ base arm reuses each pool's first completion across weight levels, overstating the effective sample size (cf.\ the H5 magnitude caveat above; directions and decisions are unaffected by the huge margins involved). The H1 and H2 specifications reported here are theory-aligned revisions of the originally frozen criteria (a fixed-$N$ weight contrast in place of a monotone-in-$N$ Brier trend for H1, an ordered-trend test in place of a regression slope for H2); both versions and the reconciliation suite that runs them side by side are archived in the repository. Multiple-comparison control is the Bonferroni--Holm procedure~\citep{holm1979simple} applied across the confirmatory family $\{H1, H2, H4, H5, H6\}$, in which H6 is the negative control and remains part of the corrected family of five; H3 is the descriptive surface-geometry observation (\S\ref{sec:exp-h3-asymptotic}), carried alongside the family with no inferential weight and not included in the correction. Mean-difference quantities use bootstrap $95\%$ confidence intervals ($10{,}000$ resamples; effect sizes as point estimates), whereas the surface-geometry violation-rate proportion of \S\ref{sec:exp-h3-asymptotic} uses an \emph{exact two-sided Clopper--Pearson} interval, since a bootstrap interval is unreliable for a bounded count statistic with few violations near zero.

\begin{table*}[htbp]
\centering
\caption{Hypotheses, falsification criteria, and results. All tests at significance level $\alpha = 0.05$ after Bonferroni--Holm correction across the family of five.}
\label{tab:hypotheses}
\footnotesize
\begin{tabularx}{\textwidth}{@{}l Y Y Y@{}}
\toprule
ID & Hypothesis & Prediction (falsified if) & Result \\
\midrule
H1 & Fixed-axis gating degradation & At fixed $N\!=\!32$, $\text{BS}(\wA/\wC\!=\!4) > \text{BS}(\wA\!=\!0)$, one-sided paired $t$-test across tasks ($\Delta \text{BS} \leq 0$) & \textbf{PASS}: $t = 10.95$, $p = 4.67\!\times\!10^{-19}$, $d = 1.10$, $\Delta \text{BS} = +0.113$ ($+41.2\%$) \\
H2 & Monotone inflation trend & On binding tasks at $N\!=\!32$, $\rmin\!=\!0.7$, mean inflation $\Delta$ is monotone non-decreasing in $\wA/\wC$, Jonckheere--Terpstra ordered-trend test (JT $z \leq 0$) & \textbf{PASS}: JT $z = 3.76$, $p = 8.5\!\times\!10^{-5}$; Spearman $\rho = 0.89$ across five weight levels \\
H4 & Threshold pooling & Under $\wA > 0$ on binding tasks, excess mass in $[\rmin, \rmin + 0.1]$, two-proportion $z$-test ($z \leq 0$) & \textbf{PASS} at all $\rmin$: $z = 30.16$ at $\rmin\!=\!0.7$ (best of three), $+47.2$\,pp excess, $p < 10^{-3}$ \\
H5 & Binding-state specificity & $\Delta|_{\Cbind} > \Delta|_{\lnot\Cbind}$ on $N\!=\!32, \wA > 0$, one-sided Welch $t$-test (ratio $\leq 1$) & \textbf{PASS}: Welch $t = 208.1$, $p < 10^{-3}$, $d = 5.35$ (at $\rmin\!=\!0.7$); per-$\rmin$ ratios $33\times/59\times/240\times$ at $\rmin\!=\!0.5/0.7/0.9$, archived under \texttt{H5\_by\_rmin} \\
H6 & Control, $\wA\!=\!0$ & Best-of-$N$ \emph{improves} calibration absent gating pressure, one-sided paired $t$-test ($\text{BS}(N\!=\!32) \geq \text{BS}(N\!=\!1)$) & \textbf{PASS}: $t = -13.08$, $p = 1.35\!\times\!10^{-23}$, $d = 1.31$, $\Delta \text{BS} = -0.277$ ($-50.3\%$) \\
\bottomrule
\end{tabularx}
\end{table*}

\subsection{\texorpdfstring{Descriptive surface geometry: midpoint interpolation by $N$}{Descriptive surface geometry: midpoint interpolation by N}}
\label{sec:exp-h3-asymptotic}

As a descriptive complement to the five hypothesis tests, we examine the geometry of the achieved $(H, C, A)$ surface as $N$ varies, via a midpoint-interpolation check: for a triple of weight vectors $(w_i, w_j, w_k)$ with $w_i < w_j < w_k$, we record whether the achieved $(H, C, A)$ at $w_j$ falls (up to a 5\% per-axis slack) below the linear interpolation between the endpoints' $(H, C, A)$. The violation rate is \emph{increasing} in $N$ (Figure~\ref{fig:h3-by-N}): $0\%$ at $N\!\in\!\{1,2\}$, $1.7\%$ at $N\!=\!4$, $13.3\%$ at $N\!=\!8$, $21.7\%$ at $N\!=\!16$, $28.3\%$ at $N\!=\!32$ (Spearman $\rho = +0.99$, $p = 3 \times 10^{-4}$). Aggregating over $N$ (one indicator per weight triple after averaging $(H, C, A)$ over the six $N$ values), 6 of the 60 triples ($\binom{6}{3} = 20$ per $\rmin$) violate: a $10\%$ aggregate rate, exact two-sided Clopper--Pearson 95\% CI $[0.038, 0.205]$. At $N\!=\!1$ selection does nothing and all weight vectors yield a single degenerate point; as $N$ grows, different weight vectors trace different operating points, and by $N\!\approx\!8$ the inflation saturation predicted by Lemma~\ref{lem:perturbation}(ii) and confirmed by H2 (inflation flat for $\wA/\wC \geq 1$) produces a plateau in the $(H, C, A)$ surface, and interpolations between weight endpoints pass above the flat. The frontier is therefore \emph{plateau-truncated} rather than concavely bowed; the rising violation rate at large $N$ describes the saturation plateau and is not evidence about the curvature of the achievable region itself.

\begin{figure}[htbp]
\centering
\includegraphics[width=0.34\linewidth]{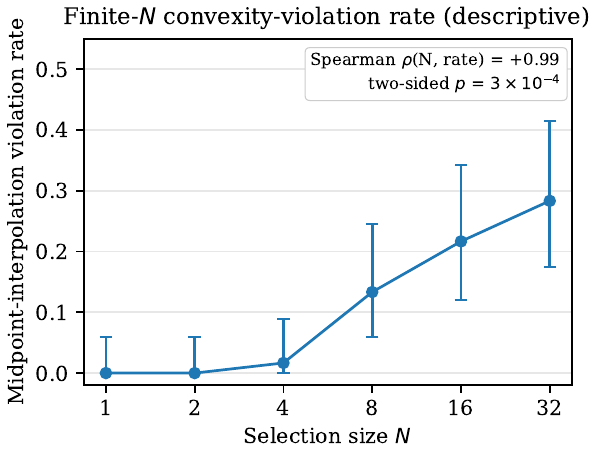}
\caption{Midpoint-interpolation violation rate by selection size $N$ (exact binomial 95\% CIs; Spearman $\rho(N, \text{rate})$ and two-sided $p$ annotated); the rising rate is the saturation-plateau signature of Prediction~2 (H2).}
\label{fig:h3-by-N}
\end{figure}

\subsection{Controls, Threats to Validity, and Reproducibility}
\label{sec:exp-controls}

Three controls isolate the gating mechanism from confounds: $\wA = 0$ (no gating reward; selection should improve calibration --- the formal content of H6 in Table~\ref{tab:hypotheses}), $N = 1$ (no selection; the base model's unperturbed calibration as the reference point), and uniform random selection among the $N$ completions (preserves the sampling distribution while removing the optimization pressure, so any systematic inflation surviving it cannot be attributed to the payoff-maximising mechanism the theorems analyse). Run as a pre-specified Stage-B control over the archived pools, random selection behaves as the mechanism account requires: every optimization-pressure signature vanishes (H1 weight contrast $d = -0.02$, $p = 0.58$; H2 trend $z = -0.48$; binding-task threshold-window mass $19.2\%$ versus the $N{=}1$ arm's $18.6\%$, $p = 0.31$, against $65.8\%$ under oracle selection; H6 $p = 0.20$), while the binding-specificity contrast H5 persists ($d = 6.47$) --- expected, since H5 compares task classes rather than selection arms and under random selection reflects the base completion law's overconfidence on hard tasks, not the payoff-maximising mechanism.

\textbf{Comonotone-selector robustness.} The H1 attribution to Proposition~\ref{prop:bon} relies on its comonotone-selection hypothesis (c), which the oracle selector of \eqref{eq:selection-payoff} does not literally instantiate: it violates empirical comonotonicity on $64.4\%$ of the multi-candidate gate-cleared selections, so the hypothesis is non-trivially binding. As a pre-specified amendment we re-ran all 540 Qwen configurations under a comonotone-compliant rearrangement selector satisfying~(c) by construction; all five Holm decisions match the primary run under the frozen min-$p$-across-thresholds family rule ($5/5$). One component of that match is structural rather than measured: under the rearrangement selector the H4 decision is carried by the $\rmin = 0.9$ window, which coincides with the entire cleared region and is therefore selector-invariant, while at the interior thresholds the pooling signature reverses sign ($z = -5.0$ and $-16.5$) as threshold pooling gives way to the top-of-range selection that Proposition~\ref{prop:bon}(ii)'s rank-aligned limit predicts. The H1, H2, H5, and H6 matches are genuine re-measurements, so the attribution does not depend on the oracle selector literally satisfying~(c).

Three threats to validity remain.
\begin{itemize}
\item \emph{Construct validity:} A model's stated confidence may not reflect an internal probability estimate. We use token-level log-probabilities \eqref{eq:logprob-confidence} as the \emph{sole} confidence metric, because they are the finest-grained internal signal the model exposes and because the verbalised alternative proved unusable: in the archived verbalised-confidence run, 93\% of selected completions report $r = 1.0$ and the remainder $r = 0.95$, independent of correctness or difficulty, leaving no variation for calibration analysis. (This degeneracy is itself the confidence-inflation behaviour the trilemma describes, but we do not lean on it inferentially; the hypothesis decisions rest on the logprob metric alone. The archived prompt template retains the verbalised-confidence elicitation line; its parsed value is stored in the pools but unused by selection.) Two caveats: the geometric-mean logprob averages over all tokens (including format and reasoning tokens), diluting the answer-specific signal, so a per-token measure focused on the answer span is future work; and the $\tau = 0.8$ softmax temperature attenuates the probability magnitudes, so a cross-temperature study is needed to confirm the hypothesis decisions are temperature-robust.
\item \emph{External validity:} The confirmatory tests were run on Qwen-2.5-7B; the architecture-independence predicted by Theorem~\ref{thm:optimizer-independence} is corroborated by the pre-specified replication on \texttt{gemma2:9b} and \texttt{yi:9b} (Appendix~\ref{app:cross-model-replication}), and generalization to substantially larger scales remains open.
\item \emph{Ecological validity:} Best-of-N is a controlled optimization, not deployment-time behavior. This is intentional: we test the theoretical mechanism, not a particular deployment scenario.
\end{itemize}

\label{sec:exp-compute}
\paragraph{Compute, code, and data availability.}
All experiments run on a single machine with Ollama serving Qwen-2.5-7B; no GPU cluster is required.
{\sloppy The full experimental pipeline is released at \url{https://github.com/Future-Computing-Group/behavioral-trilemma-experiments}, with a versioned archival snapshot deposited at Zenodo (DOI: \url{https://doi.org/10.5281/zenodo.20392749}); the repository README inventories the shipped artefacts. The implementation is \emph{two-stage}: a Stage~A pool generator (\pathtt{scripts/generate_pool.py}) samples completions with per-token log-probabilities and writes per-task pools, and a Stage~B selector (\pathtt{scripts/select_from_pool.py}) applies the oracle payoff~\eqref{eq:selection-payoff} at every $(N, \wA/\wC, \rmin, \text{seed})$ cell. Because Stage~A is a stochastic LLM-in-the-loop process whose log-probabilities depend on the \pathtt{llama.cpp} build and hardware, the pools and per-config CSVs are regenerable locally rather than shipped; the per-model held-out calibration CSVs underlying the binding sets (\S\ref{sec:exp-tasks}) and the canonical hypothesis-results JSONs (primary, rearrangement, and control suites) are shipped, Stage~B against regenerated pools is fully deterministic, and a bundled script rebuilds the hypothesis-results JSON from any run's CSVs.\par}

Inference uses the endpoint and settings of \S\ref{sec:exp-protocol} (\texttt{qwen2.5:7b}, Q4\_K\_M, $\tau = 0.8$, \texttt{logprobs: true}); the full prompt template per task category is archived alongside the driver script. Ground-truth correctness $y_i$ is determined task-type-specifically: arithmetic answers by exact computation, factual answers by exact match against the per-task ground-truth field of the committed task set (\texttt{tasks/task\_set.json}), and code outputs against test-case execution.

\section{Discussion}
\label{sec:discussion}

\subsection{Implications for Alignment Training}

\textbf{RLHF with a calibration term targets an infeasible point.} Folding a calibration signal into a standard RLHF reward~\citep{ouyang2022training} moves the system from the Goodhart regime into the trilemma regime (\S\ref{sec:goodhart-transition}, Proposition~\ref{prop:rlhf-brier}), creating incentive to \emph{appear} calibrated rather than \emph{be} calibrated; one of the three objectives must then be architecturally enforced rather than jointly optimised (Section~\ref{sec:resolution}).

\textbf{The H+A corner dominates --- conditional on Assumption~\ref{asm:cov-positive}.} RLHF reward models weight helpfulness and fluency, and human raters prefer agents that attempt to help over agents that express uncertainty~\citep{sharma2024towards, perez2023discovering}; this pressure toward the sycophancy pattern is, on the trilemma's reading, a rational response to the incentive structure rather than a training deficiency, so better training data or more RLHF iterations will not fix sycophancy while that incentive structure drives it. The strength of this conclusion depends on Assumption~\ref{asm:cov-positive}: if human raters were calibration-preferring on average, the assumption would fail and the selection-induced degradation of Proposition~\ref{prop:bon} would lose its empirical force (the impossibility theorems do not use the assumption, so the trilemma itself would still bind wherever its own assumptions hold). Recent longitudinal evidence supports the assumption's direction: in a preregistered study ($N=3{,}075$, with a three-week census-representative field component), a majority of users came to prefer sycophantic AI because it made them feel understood rather than for the quality of its advice~\citep{ibrahim2026sycophantic} --- a revealed preference for affirmation over accuracy consistent with a positive population-level helpfulness--miscalibration covariance.
\textbf{KL-constrained RL and iterative RLHF.} KL-penalised policy-gradient methods (PPO, TRPO) anchor the policy to a reference $\piref$ via $D_{\mathrm{KL}}(\pi \Vert \piref)$; two regimes must be distinguished:
\begin{itemize}
\item \emph{Fixed-reference regime.} With $\piref$ held fixed, KL-penalised RL has the closed-form optimum $\pi^*(a \mid s) \propto \piref(a \mid s) \exp(R(s, a) / \beta)$, which bounds the displacement from $\piref$ as a function of $\beta$. If $\piref$ is itself calibrated (e.g., a pre-trained base model before RLHF), a sufficiently strong KL anchor can keep the policy within a bounded-calibration-loss neighbourhood of the calibrated reference: the KL coefficient acts as a soft commitment device in the sense of~\S\ref{sec:resolution}, and the trajectory toward the shifted optimum is \emph{bounded}, not merely \emph{slowed}, as a function of $\beta$.
\item \emph{Iterative / moving-reference regime.} In iterative RLHF and online DPO-style methods, $\piref$ is updated between rounds; the fixed-policy analysis of Lemma~\ref{lem:perturbation} is then an approximation, and ``slows but does not prevent'' is the correct description: the shifted optimum is no longer a fixed target, and each reference update shifts the feasible set toward the previous round's displacement. A full two-timescale analysis~\citep{borkar1997stochastic} or mean-field treatment~\citep{lasry2007meanfield} of this setting is open.
\end{itemize}

\textbf{Context-dependent Pareto selection and regulatory alignment.} Different deployment contexts require different Pareto points: local development environments tolerate the H+A corner (low blast radius), shared infrastructure requires the H+C corner, and safety-critical systems may require the C+A corner. The EU AI Act's Article~14~\citep{eu2024aiact} mandates the H+C corner at the cost of autonomy; the NIST Risk Management Framework~\citep{nist2023rmf} aligns with calibration (``Measure'' function) and commitment (``Manage'' function).

\subsection{The Formal Model: Gradient-Trained LLMs and Information Design}
\label{sec:model-applicability}

The Perturbation Lemma models the agent as an expected-utility maximizer with a separable type $\theta$ and report $r$; Remark~\ref{rem:reduced-form} gives the reduced-form reading of this analogy for gradient-trained LLMs, which do not explicitly compute $p^*$ and decide whether to inflate. The load-bearing mechanism is population-level \emph{gradient selection}: on binding states, reports just above $\rmin$ receive the highest composite reward, and over many gradient steps this selection pressure drives the policy toward inflation as an emergent property. The trilemma's predictions are independently confirmed by the model-internal evidence reviewed in \S\ref{sec:related} and, in deployment, by the longitudinal field evidence of \citet{ibrahim2026sycophantic}.

\label{sec:info-design}
The oversight setting sits inside the information-design family but is \emph{structurally dual} to Bayesian persuasion rather than an instance of it: in \citet{kamenica2011bayesian} the uninformed sender commits to a signal structure to induce the receiver's preferred action, whereas in the CGDP the informed agent observes its own competence $p$ while the uninformed principal commits to the approval rule $q(r)$. The CGDP therefore belongs to the \emph{delegation} tradition~\citep{holmstrom1984delegation}, in which an uninformed principal commits to a decision rule over reports from an informed agent, and differs from cheap talk by scoring the report against ground truth and anchoring the gate to a single threshold.

\subsection{Beyond Reinforcement Learning: Generality of the Trilemma}
\label{sec:generality}

The method mapping of Section~\ref{sec:mapping} covers RL, RLHF, DPO, and related alignment methods, but the trilemma is not specific to them (Theorem~\ref{thm:optimizer-independence}): any ML system whose training or inference-time objective instantiates the four CGDP ingredients (private competence information, a self-reported confidence or score, a non-affine gating function conditioned on that report, and a proper scoring rule evaluating accuracy) falls within the trilemma's scope. Table~\ref{tab:generality} maps six further settings onto the four CGDP ingredients. The first four are \emph{direct} instances (a confidence-or-quality report gated against a threshold under a proper-scoring evaluation); the last two are \emph{structural analogues} in which the role mapping is inverted or the report is a contribution score rather than a confidence.

\begin{table}[htbp]
\centering
\small
\caption{Six further settings as trilemma instances; each predicted effect follows from Lemma~\ref{lem:perturbation} applied to the stated mapping.}
\label{tab:generality}
\footnotesize
\begin{tabularx}{\textwidth}{>{\hsize=1.15\hsize}X >{\hsize=0.7\hsize}X >{\hsize=0.85\hsize}X >{\hsize=1.3\hsize}X}
\toprule
Setting & Report $r$ & Gate $q$ & Predicted trilemma effect \\
\midrule
Selective prediction / learning to defer~\citep{chow1957optimum, elyaniv2010foundations, madras2018predict} & confidence score & abstention threshold & inflate near boundary to avoid deferral \\
Conformal prediction (strategic)~\citep{vovk2005algorithmic} & (de)conformity score & coverage threshold & deflate nonconformity to avoid being flagged \\
LLM tool use / function calling~\citep{schick2023toolformer} & self-assessed competence & confidence threshold for autonomous answer & inflate confidence to act without the tool \\
Mixture-of-experts routing~\citep{shazeer2017outrageously} & gating/routing score & routing selection & overconfident gating scores near each expert's competence boundary \\
Active learning (analogue)~\citep{settles2012active} & reported uncertainty & query budget & sign-reversed: under-report uncertainty when querying is penalised \\
Federated learning w/ contribution scoring (analogue)~\citep{kairouz2021advances} & local model quality & contribution-weighted inclusion & inflate quality to ensure inclusion; no mechanism gives truthful + full + helpful jointly \\
\bottomrule
\end{tabularx}
\end{table}

Conversely, the trilemma does not apply where any of the four ingredients is missing: pure supervised learning (no gate), plain RL (no self-report), GANs (the discriminator evaluates quality, not self-reported confidence), and self-supervised learning (no principal--agent structure) each lack one, and their miscalibration arises through other mechanisms (Goodhart effects, distribution shift). The trilemma is thus not an RL result that happens to generalize but a payoff-geometry result that RL happens to instantiate, applicable wherever oversight mediates a learned agent's autonomy through its self-reported confidence.

\subsection{Limitations}
\label{sec:limits}

\textbf{Modeling scope.} The stylized single rational principal does not match real deployment with intermittent monitoring, bounded rationality, or automated oversight; real oversight involves developers, deployers, users, and regulators with conflicting thresholds; and with multi-dimensional private information, the impossibility landscape may differ.

\textbf{Experimental scope.} Whether the patterns confirmed under the controlled Best-of-N selection mechanism also emerge under gradient-based training (PPO, DPO) or in-context adaptation remains an important direction for future work.

\textbf{Static reduced-form model.} The CGDP deliberately isolates the payoff geometry from the computational process (Remark~\ref{rem:reduced-form}); it does not model sequential token generation, multi-turn interaction, or in-context dynamics. An LLM may learn to inflate confidence within a single episode if the context rewards it; we conjecture an in-context inflation rate of $O(1/\sqrt{T_{\text{ctx}}})$, but formalizing this requires a theory of in-context learning for strategic behavior that does not yet exist~\citep{brown2020language, garg2022can, xie2022explanation}.

\textbf{Delegation competence.} The Commitment Resolution (Theorem~\ref{thm:commitment}) preserves helpfulness only under Assumption~\ref{asm:competent-delegation}; the impossibility direction does not depend on that assumption, but when delegate competence drops (e.g., reviewers rubber-stamp uncertain cases), achieved helpfulness degrades proportionally to the delegation rate. Characterising when the assumption holds empirically is the territory of recursive reward modelling~\citep{leike2018scalable} and sandwich evaluation~\citep{bowman2022measuring}.

\section{Conclusion}
\label{sec:conclusion}

We have established the Behavioral Credibility Trilemma: no RL policy with confidence-gated autonomy can simultaneously maximize helpfulness, calibration, and autonomy under rational oversight when some tasks exceed reliable competence; the impossibility is endogenized by the principal's own optimal gate, is optimizer-independent within the stated scope, and is confirmed on all five falsifiable predictions by the 540-configuration Best-of-N experiment of Section~\ref{sec:experiments}. Standard RL and RLHF avoid the trilemma only by not attempting calibration; once a calibration objective enters the loop, which is where the field is heading, one of the three objectives must be architecturally enforced through the commitment or separation pathways of Section~\ref{sec:resolution}. The broader point: calibrated autonomy is a multi-objective constrained optimization whose objectives are provably incompatible, and making the trade-off selection explicit is better than pretending it does not exist.


\acks{This work was conducted at the Future Computing Group, University of Oulu, Finland.}


\appendix

\section{Boundary Verification and Properness Destruction}
\label{app:proofs}

\subsection{Boundary Verification for Theorem~\ref{thm:non-affine}(ii)}
\label{app:boundary}

Parts (i)--(iii) of Theorem~\ref{thm:non-affine} are proved in Section~\ref{sec:non-affine}. Here we verify the boundary behaviour of the step rule $q^*(r) = \ind\{r \geq r_0\}$, $r_0 = p_{\min} + \sqrt{\wA R^*/\wC}$, underlying the first-best screening claim of part~(ii).

\emph{Uniqueness of $r_0$ as best response for pooling types.} For a type $p \in [p_{\min}, r_0)$, the payoff from reporting $r \geq r_0$ is $V(r) = -\wC(r - p)^2 + \wA R^*$. This is strictly decreasing in $r$ for $r > p$, so the unique optimum among approved reports is $r = r_0$ (the lowest approved report). Any $r \in (p, r_0)$ yields $V(r) = -\wC(r-p)^2 < 0 < V(r_0)$ when $\wA R^* > \wC(r_0 - p)^2$, which holds for $p \geq p_{\min}$ by construction.

\emph{Types $p < p_{\min}$.} Type $p = p_{\min} - \varepsilon$ gets utility $V(r_0) = -\wC(r_0 - p)^2 + \wA R^* = -\wC(\sqrt{\wA R^*/\wC} + \varepsilon)^2 + \wA R^* = -2\varepsilon\sqrt{\wC\wA R^*} - \wC\varepsilon^2 < 0$ from inflating, so truthful reporting ($V(p) = 0$) is strictly preferred.

\emph{Types $p > r_0$.} These types satisfy $p > r_0 > p_{\min}$, so truthful reporting $r = p$ clears the gate: $V(p) = \wA R^* > 0$. Any report $r \neq p$ incurs a calibration penalty $-\wC(r - p)^2 < 0$, so $r = p$ is the unique optimum. These types are approved and report truthfully.

\emph{Tie-breaking at $p = p_{\min}$.} The marginal type $p = p_{\min}$ is exactly indifferent: $V(r_0) = -\wC \cdot \wA R^*/\wC + \wA R^* = 0 = V(p_{\min})$. We adopt the standard tie-breaking convention favoring the mechanism, so $p_{\min}$ pools at $r_0$ and is approved, implementing the first-best boundary.

\subsection{Properness Destruction by Report-Dependent Perturbations}

\begin{proposition}[Properness under Report-Dependent Perturbation]
\label{prop:properness-destruction}
Let $S(r, \omega)$ be a strictly proper scoring rule with convex generator $G \in C^2((0,1))$. Let $h: [0,1] \to \R$ be $C^1$. The perturbed rule $\tilde{S}(r, \omega) = S(r, \omega) + h(r)$ is strictly proper if and only if $h$ is constant.
\end{proposition}

\begin{proof}
The expected perturbed score is $\E_p[\tilde{S}] = G(r) + G'(r)(p - r) + h(r)$. Differentiating with respect to $r$:
\[
\frac{\partial}{\partial r}\E_p[\tilde{S}] = G'(r) + G''(r)(p - r) - G'(r) + h'(r) = G''(r)(p - r) + h'(r),
\]
where the $G'(r)$ terms cancel. Setting $r = p$: $h'(p) = 0$ for all $p \in (0,1)$. Since $h$ is $C^1$, $h' \equiv 0$, so $h$ is constant. The converse is immediate: a constant $h$ shifts the expected score by a constant, leaving the argmax and strict properness unchanged.
\end{proof}

\begin{corollary}[Gated Autonomy Perturbations]
\label{cor:gate-destroys-properness}
Suppose $\wA R(s,a) > 0$. Then the autonomy perturbation $h(r) = \wA \cdot q(r) \cdot R(s,a)$ destroys strict properness whenever $q$ is non-constant. For $q \in C^1$ this is immediate from Proposition~\ref{prop:properness-destruction}, since $h$ is then $C^1$ and non-constant. For the canonical step gate $q(r) = \ind\{r \geq \rmin\}$ with interior threshold $\rmin \in (0, 1)$, which is not $C^1$, the conclusion follows instead from the threshold comparison of Lemma~\ref{lem:perturbation}(ii) (with $\wC = 1$, since the perturbed rule is $S + h$): for $p < \rmin$ the truthful report forfeits the gate bonus $\wA R(s,a) > 0$, while inflating to $r = \rmin$ costs only the calibration penalty $G(p) - G(\rmin) - G'(\rmin)(p - \rmin)$, which vanishes as $p \to \rmin^-$; hence types sufficiently close below the threshold strictly prefer $r = \rmin$ (under the Brier score, exactly the types $p \in (\rmin - \sqrt{\wA R(s,a)},\, \rmin)$), and $\tilde{S}$ is not strictly proper. The hypothesis $\wA R(s,a) > 0$ is necessary: for the zero-reward abstain action of Assumption~\ref{asm:abstain}, $h \equiv 0$ regardless of $q$, and properness is unaffected.
\end{corollary}

\section{Cross-Model Evidence}
\label{app:cross-model-placement}

\subsection{Gated cross-model replication}
\label{app:cross-model-replication}

The external-validity limitation noted in \S\ref{sec:exp-models} and \S\ref{sec:limits} is addressed by re-running the complete 540-configuration gated protocol on two additional open-weights families, \texttt{gemma2:9b} (Google Gemma~2) and \texttt{yi:9b} (01.AI Yi), under the pre-specified design, with the per-model held-out binding-set estimates (\S\ref{sec:exp-tasks}) and the five confirmatory Holm decisions fixed in advance (\texttt{yi:9b} is the pre-specified fallback slot, entered after the frozen second slot, \texttt{mistral:7b}, tripped the pre-specified parse-rate gate); the mechanism is declared replicated on a family iff all five decisions match the primary. The binding-set counts in Table~\ref{tab:cross-model-replication} order the three families from most to least capable on this task set, so reproduction across that range shows the trilemma's empirical signature is not an artefact of one model's competence profile.

All five Holm decisions match the primary on both new families (Table~\ref{tab:cross-model-replication}, $5/5$ each). H1 and H4--H6 reproduce at large margins; the monotone-inflation-trend test H2, though significant on both new families (raw $p \approx 3 \times 10^{-3}$ on each), does so at a more modest margin than on the primary, making the ordered-trend signal the weakest of the five away from the primary model. Each family ran entirely on a single compute host (\texttt{yi:9b} on a CSC Puhti V100 node, the others on local Apple-Silicon Ollama), preserving per-model internal consistency.

\begin{table*}[htbp]
\centering
\caption{Pre-specified gated cross-model replication of the five confirmatory hypotheses (full 540-configuration protocol per family; statistics at $N\!=\!32$, $\rmin\!=\!0.7$ as in Table~\ref{tab:hypotheses}). Every decision is REJECT under the Bonferroni--Holm correction across the five.}
\label{tab:cross-model-replication}
\footnotesize
\begin{tabularx}{\textwidth}{@{}l Y Y Y@{}}
\toprule
 & \texttt{qwen2.5:7b} (primary) & \texttt{gemma2:9b} & \texttt{yi:9b} \\
\midrule
Binding sets ($\rmin\!=\!0.5/0.7/0.9$) & $71/72/74$ & $59/61/65$ & $84/88/97$ \\
Stage-A unparsed rate & --- & $0.87\%$ & $0.51\%$ \\
\midrule
H1 (gating degradation) & $t = 10.95$, $d = 1.10$ & $t = 7.64$, $d = 0.76$ & $t = 10.58$, $d = 1.06$ \\
H2 (monotone inflation) & $z = 3.76$ & $z = 2.76$, $p = 2.9\!\times\!10^{-3}$ & $z = 2.71$, $p = 3.4\!\times\!10^{-3}$ \\
H4 (threshold pooling) & $z = 30.16$ & $z = 26.4$ & $z = 14.9$ \\
H5 (binding specificity) & $d = 5.35$ & $d = 6.37$ & $d = 3.84$ \\
H6 (control, $\wA\!=\!0$) & $t = -13.08$ & $t = -10.9$ & $t = -21.0$ \\
\midrule
Holm decision (5/5) & \textbf{primary} & \textbf{REPLICATED} & \textbf{REPLICATED} \\
\bottomrule
\end{tabularx}
\end{table*}

\subsection{Descriptive placement}
\label{app:cross-model-descriptive}

\begin{figure}[htbp]
\centering
\includegraphics[width=0.44\linewidth]{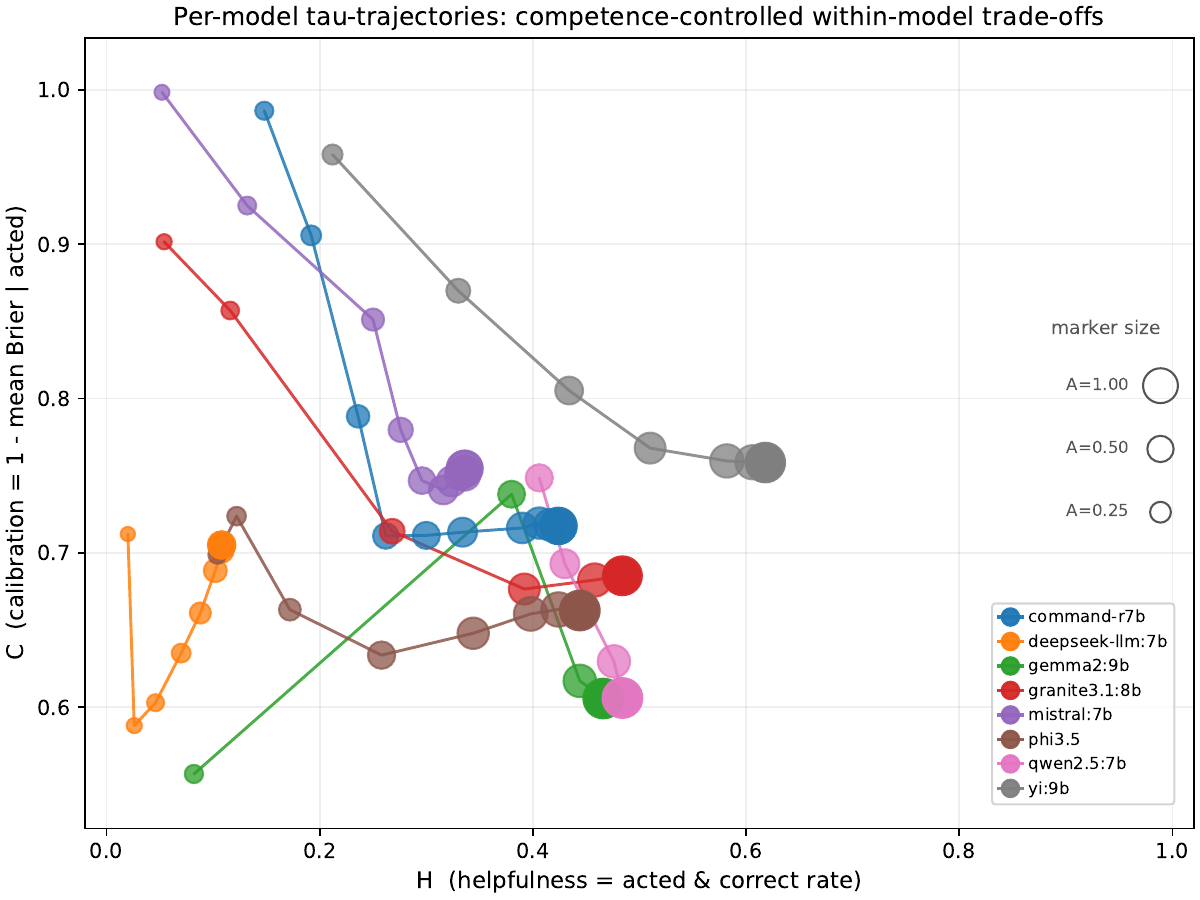}
\caption{Per-model $\tau$-trajectories on the $(H, C)$ behavioral axes (marker size $\propto A$, action rate); baseline position reflects competence, curve shape the within-model trade-off. No model approaches the joint-good corner $(H, C, A)\!\to\!1$.}
\label{fig:cross-model-logprob}
\end{figure}

Figure~\ref{fig:cross-model-logprob} places eight open-weights architectures on the $(H, C)$ behavioral axes under an ungated protocol, a comparability anchor superseded by the gated replication above for any inferential purpose. The eight architectures (named in the figure legend) run through one OpenAI-compatible logprob client at $N\!=\!1$ ungated (100-task subset, five seeds); each is rendered as a $\tau$-trajectory in $(H, C)$ (predictions with $r_{\text{logprob}} \!<\! \tau$ reclassified as abstained as $\tau$ sweeps $0$ to $1$).

Three caveats bound this anchor: the autonomy axis conflates deliberate deferral with parse failures (\texttt{mistral:7b-instruct}'s $A\!\approx\!0.79$ is parse-fail-driven), the calibration value is a faithful but not byte-exact reimplementation of \eqref{eq:logprob-confidence}, and correctness uses a charitable verifier (\texttt{analysis.robust\_verify}) to absorb cross-family verbosity differences, so the figure's \texttt{qwen2.5:7b} point will not numerically match the gated 540-config numbers.

\bibliography{references}

\begin{thebibliography}{95}
\providecommand{\natexlab}[1]{#1}
\providecommand{\url}[1]{\texttt{#1}}
\expandafter\ifx\csname urlstyle\endcsname\relax
  \providecommand{\doi}[1]{doi: #1}\else
  \providecommand{\doi}{doi: \begingroup \urlstyle{rm}\Url}\fi

\bibitem[Achiam et~al.(2017)Achiam, Held, Tamar, and
  Abbeel]{achiam2017constrained}
Joshua Achiam, David Held, Aviv Tamar, and Pieter Abbeel.
\newblock Constrained policy optimization.
\newblock In \emph{Proceedings of the 34th International Conference on Machine
  Learning}, volume~70 of \emph{PMLR}, pages 22--31, 2017.

\bibitem[Akbarpour and Li(2020)]{akbarpour2020credible}
Mohammad Akbarpour and Shengwu Li.
\newblock Credible auctions: A trilemma.
\newblock \emph{Econometrica}, 88\penalty0 (2):\penalty0 425--467, 2020.

\bibitem[{Anthropic}(2026)]{anthropic2026claudeopus}
{Anthropic}.
\newblock Claude {Opus} 4.6 system card.
\newblock \url{https://www.anthropic.com/claude-opus-4-6-system-card}, 2026.
\newblock System card.

\bibitem[Bai et~al.(2022)Bai, Kadavath, Kundu, Askell, Kernion, Jones, Chen,
  Goldie, Mirhoseini, McKinnon, et~al.]{bai2022constitutional}
Yuntao Bai, Saurav Kadavath, Sandipan Kundu, Amanda Askell, Jackson Kernion,
  Andy Jones, Anna Chen, Anna Goldie, Azalia Mirhoseini, Cameron McKinnon,
  et~al.
\newblock Constitutional {AI}: Harmlessness from {AI} feedback.
\newblock Technical report, Anthropic, 2022.

\bibitem[Borkar(1997)]{borkar1997stochastic}
Vivek~S. Borkar.
\newblock Stochastic approximation with two time scales.
\newblock \emph{Systems \& Control Letters}, 29\penalty0 (5):\penalty0
  291--294, 1997.

\bibitem[Bowman et~al.(2022)Bowman, Hyun, Perez, Chen, Pettit, Heiner,
  Lukosiute, Askell, Jones, Chen, et~al.]{bowman2022measuring}
Samuel~R. Bowman, Jeeyoon Hyun, Ethan Perez, Edwin Chen, Craig Pettit, Scott
  Heiner, Kamile Lukosiute, Amanda Askell, Andy Jones, Anna Chen, et~al.
\newblock Measuring progress on scalable oversight for large language models.
\newblock \emph{arXiv preprint arXiv:2211.03540}, 2022.

\bibitem[Bradshaw et~al.(2003)Bradshaw, Feltovich, Jung, Kulkarni, Taysom, and
  Uszok]{bradshaw2003human}
Jeffrey~M Bradshaw, Paul~J Feltovich, Hyuckchul Jung, Shriniwas Kulkarni,
  William Taysom, and Andrzej Uszok.
\newblock Human-agent teamwork and adjustable autonomy in practice.
\newblock In \emph{Proceedings of the 7th International Symposium on Artificial
  Intelligence, Robotics and Automation in Space}, 2003.

\bibitem[Brier(1950)]{brier1950verification}
Glenn~W Brier.
\newblock Verification of forecasts expressed in terms of probability.
\newblock \emph{Monthly Weather Review}, 78\penalty0 (1):\penalty0 1--3, 1950.

\bibitem[Brown et~al.(2020)Brown, Mann, Ryder, Subbiah, Kaplan, Dhariwal,
  Neelakantan, Shyam, Sastry, Askell, et~al.]{brown2020language}
Tom~B Brown, Benjamin Mann, Nick Ryder, Melanie Subbiah, Jared~D Kaplan,
  Prafulla Dhariwal, Arvind Neelakantan, Pranav Shyam, Girish Sastry, Amanda
  Askell, et~al.
\newblock Language models are few-shot learners.
\newblock In \emph{Advances in Neural Information Processing Systems},
  volume~33, pages 1877--1901, 2020.

\bibitem[Casper et~al.(2023)Casper, Davies, Shi, Gilbert, Scheurer, Rando,
  Freedman, Korbak, Lindner, et~al.]{casper2023open}
Stephen Casper, Xander Davies, Claudia Shi, Thomas~Krendl Gilbert,
  J{\'e}r{\'e}my Scheurer, Javier Rando, Rachel Freedman, Tomasz Korbak, David
  Lindner, et~al.
\newblock Open problems and fundamental limitations of reinforcement learning
  from human feedback.
\newblock \emph{Transactions on Machine Learning Research}, 2023.
\newblock URL \url{https://openreview.net/forum?id=bx24KpJ4Eb}.

\bibitem[Chen and Pennock(2010)]{chen2010designing}
Yiling Chen and David~M Pennock.
\newblock Designing markets for prediction.
\newblock \emph{AI Magazine}, 31\penalty0 (4):\penalty0 42--52, 2010.

\bibitem[Chen et~al.(2014)Chen, Kash, Ruberry, and Shnayder]{chen2014eliciting}
Yiling Chen, Ian~A. Kash, Michael Ruberry, and Victor Shnayder.
\newblock Eliciting predictions and recommendations for decision making.
\newblock \emph{ACM Transactions on Economics and Computation}, 2\penalty0
  (2):\penalty0 6:1--6:27, 2014.
\newblock \doi{10.1145/2556271}.

\bibitem[Chow(1957)]{chow1957optimum}
C.~K. Chow.
\newblock An optimum character recognition system using decision functions.
\newblock \emph{IRE Transactions on Electronic Computers}, EC-6\penalty0
  (4):\penalty0 247--254, 1957.
\newblock \doi{10.1109/TEC.1957.5222035}.

\bibitem[Christiano et~al.(2017)Christiano, Leike, Brown, Martic, Legg, and
  Amodei]{christiano2017deep}
Paul~F. Christiano, Jan Leike, Tom Brown, Miljan Martic, Shane Legg, and Dario
  Amodei.
\newblock Deep reinforcement learning from human preferences.
\newblock In \emph{Advances in Neural Information Processing Systems 30
  (NeurIPS 2017)}, pages 4299--4307, 2017.

\bibitem[Clarke(1971)]{clarke1971multipart}
Edward~H. Clarke.
\newblock Multipart pricing of public goods.
\newblock \emph{Public Choice}, 11:\penalty0 17--33, 1971.

\bibitem[Conitzer et~al.(2024)Conitzer, Freedman, Heitzig, Holliday, Jacobs,
  Lambert, Moss{\'e}, Pacuit, Russell, Schoelkopf, Tewolde, and
  Zwicker]{conitzer2024social}
Vincent Conitzer, Rachel Freedman, Jobst Heitzig, Wesley~H Holliday, Bob~M
  Jacobs, Nathan Lambert, Milan Moss{\'e}, Eric Pacuit, Stuart Russell, Hailey
  Schoelkopf, Emanuel Tewolde, and William~S Zwicker.
\newblock Position: Social choice should guide {AI} alignment in dealing with
  diverse human feedback.
\newblock In \emph{Proceedings of the 41st International Conference on Machine
  Learning}, pages 9346--9360, 2024.

\bibitem[Crawford and Sobel(1982)]{crawford1982strategic}
Vincent~P Crawford and Joel Sobel.
\newblock Strategic information transmission.
\newblock \emph{Econometrica}, 50\penalty0 (6):\penalty0 1431--1451, 1982.

\bibitem[Dai et~al.(2024)Dai, Pan, Sun, Ji, Xu, Liu, Wang, and
  Yang]{dai2024safe}
Josef Dai, Xuehai Pan, Ruiyang Sun, Jiaming Ji, Xinbo Xu, Mickel Liu, Yizhou
  Wang, and Yaodong Yang.
\newblock Safe {RLHF}: Safe reinforcement learning from human feedback.
\newblock In \emph{Proceedings of the International Conference on Learning
  Representations}, 2024.

\bibitem[Dawid(1982)]{dawid1982well}
A.~Philip Dawid.
\newblock The well-calibrated {B}ayesian.
\newblock \emph{Journal of the American Statistical Association}, 77\penalty0
  (379):\penalty0 605--610, 1982.

\bibitem[Desai and Durrett(2020)]{desai2020calibration}
Shrey Desai and Greg Durrett.
\newblock Calibration of pre-trained transformers.
\newblock In \emph{Proceedings of the 2020 Conference on Empirical Methods in
  Natural Language Processing}, pages 295--302, 2020.

\bibitem[Dorais et~al.(1999)Dorais, Bonasso, Kortenkamp, Pell, and
  Schreckenghost]{dorais1999adjustable}
Gregory~A Dorais, R~Peter Bonasso, David Kortenkamp, Barney Pell, and Debra
  Schreckenghost.
\newblock Adjustable autonomy for human-centered autonomous systems on {Mars}.
\newblock In \emph{AAAI Spring Symposium on Agents with Adjustable Autonomy},
  AAAI Technical Report SS-99-06, pages 16--35, 1999.

\bibitem[El-Yaniv and Wiener(2010)]{elyaniv2010foundations}
Ran El-Yaniv and Yair Wiener.
\newblock On the foundations of noise-free selective classification.
\newblock \emph{Journal of Machine Learning Research}, 11:\penalty0 1605--1641,
  2010.

\bibitem[Embrechts et~al.(2002)Embrechts, McNeil, and
  Straumann]{embrechts2002quantile}
Paul Embrechts, Alexander McNeil, and Daniel Straumann.
\newblock Correlation and dependence in risk management: properties and
  pitfalls.
\newblock In M.~A.~H. Dempster, editor, \emph{Risk Management: Value at Risk
  and Beyond}, pages 176--223. Cambridge University Press, 2002.

\bibitem[Ethayarajh et~al.(2024)Ethayarajh, Xu, Muennighoff, Jurafsky, and
  Kiela]{ethayarajh2024kto}
Kawin Ethayarajh, Winnie Xu, Niklas Muennighoff, Dan Jurafsky, and Douwe Kiela.
\newblock {KTO}: Model alignment as prospect theoretic optimization.
\newblock In \emph{Proceedings of the 41st International Conference on Machine
  Learning (ICML 2024)}, 2024.

\bibitem[{European Parliament and Council of the European
  Union}(2024)]{eu2024aiact}
{European Parliament and Council of the European Union}.
\newblock Regulation ({EU}) 2024/1689 of the {European} parliament and of the
  council of 13 june 2024 laying down harmonised rules on artificial
  intelligence ({Artificial Intelligence Act}), 2024.
\newblock OJ L 2024/1689, 12.7.2024.

\bibitem[Fortuin et~al.(1971)Fortuin, Kasteleyn, and
  Ginibre]{fortuin1971correlation}
C.~M. Fortuin, P.~W. Kasteleyn, and J.~Ginibre.
\newblock Correlation inequalities on some partially ordered sets.
\newblock \emph{Communications in Mathematical Physics}, 22\penalty0
  (2):\penalty0 89--103, 1971.

\bibitem[Frongillo and Kash(2015)]{frongillo2015vector}
Rafael Frongillo and Ian Kash.
\newblock Vector-valued property elicitation.
\newblock In \emph{Proceedings of The 28th Conference on Learning Theory},
  pages 846--865, 2015.

\bibitem[Fudenberg and Tirole(1991)]{fudenberg1991game}
Drew Fudenberg and Jean Tirole.
\newblock \emph{Game Theory}.
\newblock MIT Press, 1991.

\bibitem[Gaikwad(2025)]{gaikwad2025laws}
Madhava Gaikwad.
\newblock Murphy's laws of {AI} alignment: Why the gap always wins.
\newblock \emph{arXiv preprint arXiv:2509.05381}, 2025.
\newblock Preprint.

\bibitem[Gao et~al.(2022)Gao, Schulman, and Hilton]{gao2023scaling}
Leo Gao, John Schulman, and Jacob Hilton.
\newblock Scaling laws for reward model overoptimization.
\newblock \emph{arXiv preprint arXiv:2210.10760}, 2022.

\bibitem[Garg et~al.(2022)Garg, Tsipras, Liang, and Valiant]{garg2022can}
Shivam Garg, Dimitris Tsipras, Percy Liang, and Gregory Valiant.
\newblock What can transformers learn in-context? {A} case study of simple
  function classes.
\newblock In \emph{Advances in Neural Information Processing Systems},
  volume~35, pages 30583--30598, 2022.

\bibitem[Gneiting and Raftery(2007)]{gneiting2007strictly}
Tilmann Gneiting and Adrian~E Raftery.
\newblock Strictly proper scoring rules, prediction, and estimation.
\newblock \emph{Journal of the American Statistical Association}, 102\penalty0
  (477):\penalty0 359--378, 2007.
\newblock \doi{10.1198/016214506000001437}.

\bibitem[Groves(1973)]{groves1973incentives}
Theodore Groves.
\newblock Incentives in teams.
\newblock \emph{Econometrica}, 41\penalty0 (4):\penalty0 617--631, 1973.

\bibitem[Guo et~al.(2017)Guo, Pleiss, Sun, and Weinberger]{guo2017calibration}
Chuan Guo, Geoff Pleiss, Yu~Sun, and Kilian~Q Weinberger.
\newblock On calibration of modern neural networks.
\newblock In \emph{Proceedings of the 34th International Conference on Machine
  Learning}, pages 1321--1330, 2017.

\bibitem[Hadfield-Menell and Hadfield(2019)]{hadfield2019incomplete}
Dylan Hadfield-Menell and Gillian~K Hadfield.
\newblock Incomplete contracting and {AI} alignment.
\newblock In \emph{Proceedings of the 2019 AAAI/ACM Conference on AI, Ethics,
  and Society}, pages 417--422, 2019.

\bibitem[Hanson(2003)]{hanson2003combinatorial}
Robin Hanson.
\newblock Combinatorial information market design.
\newblock \emph{Information Systems Frontiers}, 5\penalty0 (1):\penalty0
  107--119, 2003.

\bibitem[Hardt et~al.(2016)Hardt, Megiddo, Papadimitriou, and
  Wootters]{hardt2016strategic}
Moritz Hardt, Nimrod Megiddo, Christos Papadimitriou, and Mary Wootters.
\newblock Strategic classification.
\newblock In \emph{Proceedings of the 2016 ACM Conference on Innovations in
  Theoretical Computer Science (ITCS)}, pages 111--122, 2016.
\newblock \doi{10.1145/2840728.2840730}.

\bibitem[Harris(1960)]{harris1960lower}
T.~E. Harris.
\newblock A lower bound for the critical probability in a certain percolation
  process.
\newblock \emph{Proceedings of the Cambridge Philosophical Society},
  56\penalty0 (1):\penalty0 13--20, 1960.

\bibitem[Hart and Moore(1988)]{hart1988incomplete}
Oliver Hart and John Moore.
\newblock Incomplete contracts and renegotiation.
\newblock \emph{Econometrica}, 56\penalty0 (4):\penalty0 755--785, 1988.

\bibitem[Hoeffding(1940)]{hoeffding1940masstabinvariante}
Wassily Hoeffding.
\newblock Masstabinvariante korrelationstheorie.
\newblock \emph{Schriften des Mathematischen Instituts und des Instituts f\"ur
  Angewandte Mathematik der Universit\"at Berlin}, 5:\penalty0 181--233, 1940.

\bibitem[Holm(1979)]{holm1979simple}
Sture Holm.
\newblock A simple sequentially rejective multiple test procedure.
\newblock \emph{Scandinavian Journal of Statistics}, 6\penalty0 (2):\penalty0
  65--70, 1979.

\bibitem[Holmstr\"{o}m(1984)]{holmstrom1984delegation}
Bengt Holmstr\"{o}m.
\newblock On the theory of delegation.
\newblock \emph{Bayesian Models in Economic Theory (M.~Boyer and R.~Kihlstrom,
  eds.), North-Holland}, pages 115--141, 1984.

\bibitem[Ibrahim et~al.(2026)Ibrahim, Hafner, Cheng, Lee, Anselmetti, Willer,
  Rocher, and Yang]{ibrahim2026sycophantic}
Lujain Ibrahim, Franziska~Sofia Hafner, Myra Cheng, Cinoo Lee, Rebecca
  Anselmetti, Robb Willer, Luc Rocher, and Diyi Yang.
\newblock Sycophantic {AI} makes human interaction feel more effortful and less
  satisfying over time.
\newblock \emph{arXiv preprint arXiv:2605.07912}, 2026.

\bibitem[Jiang et~al.(2021)Jiang, Araki, Ding, and Neubig]{jiang2021can}
Zhengbao Jiang, Jun Araki, Haibo Ding, and Graham Neubig.
\newblock How can we know when language models know? on the calibration of
  language models for question answering.
\newblock \emph{Transactions of the Association for Computational Linguistics},
  9:\penalty0 962--977, 2021.

\bibitem[Kadavath et~al.(2022)Kadavath, Conerly, Askell, Henighan, Drain,
  Perez, Schiefer, Hatfield-Dodds, DasSarma, Tran-Johnson,
  et~al.]{kadavath2022language}
Saurav Kadavath, Tom Conerly, Amanda Askell, Tom Henighan, Dawn Drain, Ethan
  Perez, Nicholas Schiefer, Zac Hatfield-Dodds, Nova DasSarma, Eli
  Tran-Johnson, et~al.
\newblock Language models (mostly) know what they know.
\newblock \emph{arXiv preprint arXiv:2207.05221}, 2022.

\bibitem[Kairouz et~al.(2021)Kairouz, McMahan, Avent, Bellet, Bennis, Bhagoji,
  Bonawitz, Charles, Cormode, Cummings, et~al.]{kairouz2021advances}
Peter Kairouz, H.~Brendan McMahan, Brendan Avent, Aur{\'e}lien Bellet, Mehdi
  Bennis, Arjun~Nitin Bhagoji, Kallista Bonawitz, Zachary Charles, Graham
  Cormode, Rachel Cummings, et~al.
\newblock Advances and open problems in federated learning.
\newblock \emph{Foundations and Trends in Machine Learning}, 14\penalty0
  (1--2):\penalty0 1--210, 2021.

\bibitem[Kamenica and Gentzkow(2011)]{kamenica2011bayesian}
Emir Kamenica and Matthew Gentzkow.
\newblock Bayesian persuasion.
\newblock \emph{American Economic Review}, 101\penalty0 (6):\penalty0
  2590--2615, 2011.
\newblock \doi{10.1257/aer.101.6.2590}.

\bibitem[Laffont and Tirole(1993)]{laffont1993theory}
Jean-Jacques Laffont and Jean Tirole.
\newblock \emph{A Theory of Incentives in Procurement and Regulation}.
\newblock MIT Press, 1993.

\bibitem[Lang et~al.(2024)Lang, Foote, Russell, Dragan, Jenner, and
  Emmons]{lang2024theoretical}
Leon Lang, Davis Foote, Stuart Russell, Anca Dragan, Erik Jenner, and Scott
  Emmons.
\newblock When your {AI}s deceive you: Challenges of partial observability in
  reinforcement learning from human feedback.
\newblock In \emph{Advances in Neural Information Processing Systems 37
  ({NeurIPS} 2024)}, 2024.

\bibitem[Lasry and Lions(2007)]{lasry2007meanfield}
Jean-Michel Lasry and Pierre-Louis Lions.
\newblock Mean field games.
\newblock \emph{Japanese Journal of Mathematics}, 2\penalty0 (1):\penalty0
  229--260, 2007.

\bibitem[Leike et~al.(2018)Leike, Krueger, Everitt, Martic, Maini, and
  Legg]{leike2018scalable}
Jan Leike, David Krueger, Tom Everitt, Miljan Martic, Vishal Maini, and Shane
  Legg.
\newblock Scalable agent alignment via reward modeling: a research direction.
\newblock \emph{arXiv preprint arXiv:1811.07871}, 2018.

\bibitem[Lov\'{e}n(2026)]{loven2026honestreporting}
Lauri Lov\'{e}n.
\newblock Honest reporting in scored oversight: {The} {True-KL$_0$} {Property}
  via the {Pr\'{e}kopa} {Principle}.
\newblock \emph{arXiv preprint arXiv:2605.03793}, 2026.

\bibitem[Lov\'{e}n and Tarkoma(2026)]{loven_endogeneity_2026}
Lauri Lov\'{e}n and Sasu Tarkoma.
\newblock The endogeneity of miscalibration: Impossibility and escape in scored
  reporting.
\newblock \emph{arXiv preprint arXiv:2605.07671}, 2026.

\bibitem[Madras et~al.(2018)Madras, Pitassi, and Zemel]{madras2018predict}
David Madras, Toniann Pitassi, and Richard Zemel.
\newblock Predict responsibly: Improving fairness and accuracy by learning to
  defer.
\newblock In \emph{Advances in Neural Information Processing Systems},
  volume~31, 2018.

\bibitem[Manheim and Garrabrant(2019)]{manheim2019categorizing}
David Manheim and Scott Garrabrant.
\newblock Categorizing variants of {Goodhart's Law}.
\newblock \emph{arXiv preprint arXiv:1803.04585}, 2019.

\bibitem[McCarthy(1956)]{mccarthy1956measures}
John McCarthy.
\newblock Measures of the value of information.
\newblock \emph{Proceedings of the National Academy of Sciences}, 42\penalty0
  (9):\penalty0 654--655, 1956.

\bibitem[Mirrlees(1971)]{mirrlees1971exploration}
James~A. Mirrlees.
\newblock An exploration in the theory of optimum income taxation.
\newblock \emph{The Review of Economic Studies}, 38\penalty0 (2):\penalty0
  175--208, 1971.
\newblock \doi{10.2307/2296779}.

\bibitem[Murphy(1973)]{murphy1973new}
Allan~H. Murphy.
\newblock A new vector partition of the probability score.
\newblock \emph{Journal of Applied Meteorology}, 12\penalty0 (4):\penalty0
  595--600, 1973.

\bibitem[Myerson(1981)]{myerson1981optimal}
Roger~B Myerson.
\newblock Optimal auction design.
\newblock \emph{Mathematics of Operations Research}, 6\penalty0 (1):\penalty0
  58--73, 1981.

\bibitem[Naeini et~al.(2015)Naeini, Cooper, and
  Hauskrecht]{naeini2015obtaining}
Mahdi~Pakdaman Naeini, Gregory Cooper, and Milos Hauskrecht.
\newblock Obtaining well calibrated probabilities using {B}ayesian binning into
  quantiles.
\newblock In \emph{Proceedings of the AAAI Conference on Artificial
  Intelligence}, volume~29, 2015.
\newblock \doi{10.1609/aaai.v29i1.9602}.

\bibitem[Nakano et~al.(2021)Nakano, Hilton, Balaji, Wu, Ouyang, Kim, Hesse,
  Jain, Kosaraju, Saunders, et~al.]{nakano2021webgpt}
Reiichiro Nakano, Jacob Hilton, Suchir Balaji, Jeff Wu, Long Ouyang, Christina
  Kim, Christopher Hesse, Shantanu Jain, Vineet Kosaraju, William Saunders,
  et~al.
\newblock {WebGPT}: Browser-assisted question-answering with human feedback.
\newblock arXiv preprint arXiv:2112.09332, 2021.

\bibitem[{National Institute of Standards and Technology}(2023)]{nist2023rmf}
{National Institute of Standards and Technology}.
\newblock Artificial intelligence risk management framework ({AI RMF} 1.0).
\newblock \emph{NIST AI 100-1}, 2023.

\bibitem[Oesterheld and Conitzer(2020)]{oesterheld2020decision}
Caspar Oesterheld and Vincent Conitzer.
\newblock Decision scoring rules.
\newblock In \emph{Web and Internet Economics ({WINE} 2020)}, volume 12495 of
  \emph{Lecture Notes in Computer Science}, page 468. Springer, 2020.
\newblock One-page abstract; extended version available from the authors.

\bibitem[Oesterheld et~al.(2023)Oesterheld, Treutlein, Cooper, and
  Hudson]{oesterheld2023incentivizing}
Caspar Oesterheld, Johannes Treutlein, Emery Cooper, and Rubi Hudson.
\newblock Incentivizing honest performative predictions with proper scoring
  rules.
\newblock In \emph{Proceedings of the Thirty-Ninth Conference on Uncertainty in
  Artificial Intelligence ({UAI})}, volume 216 of \emph{Proceedings of Machine
  Learning Research}, pages 1564--1574. PMLR, 2023.

\bibitem[Othman and Sandholm(2010)]{othman2010decision}
Abraham Othman and Tuomas Sandholm.
\newblock Decision rules and decision markets.
\newblock In \emph{Proceedings of the 9th International Conference on
  Autonomous Agents and Multiagent Systems ({AAMAS})}, pages 625--632, 2010.

\bibitem[Ouyang et~al.(2022)Ouyang, Wu, Jiang, Almeida, Wainwright, Mishkin,
  Zhang, Agarwal, Slama, Ray, et~al.]{ouyang2022training}
Long Ouyang, Jeffrey Wu, Xu~Jiang, Diogo Almeida, Carroll Wainwright, Pamela
  Mishkin, Chong Zhang, Sandhini Agarwal, Katarina Slama, Alex Ray, et~al.
\newblock Training language models to follow instructions with human feedback.
\newblock \emph{Advances in Neural Information Processing Systems}, 35, 2022.

\bibitem[Ovadia et~al.(2019)Ovadia, Fertig, Ren, Nado, Sculley, Nowozin,
  Dillon, Lakshminarayanan, and Snoek]{ovadia2019can}
Yaniv Ovadia, Emily Fertig, Jie Ren, Zachary Nado, David Sculley, Sebastian
  Nowozin, Joshua~V Dillon, Balaji Lakshminarayanan, and Jasper Snoek.
\newblock Can you trust your model's uncertainty? {E}valuating predictive
  uncertainty under dataset shift.
\newblock \emph{Advances in Neural Information Processing Systems}, 32, 2019.

\bibitem[Parasuraman et~al.(2000)Parasuraman, Sheridan, and
  Wickens]{parasuraman2000model}
Raja Parasuraman, Thomas~B Sheridan, and Christopher~D Wickens.
\newblock A model for types and levels of human interaction with automation.
\newblock \emph{IEEE Transactions on Systems, Man, and Cybernetics-Part A:
  Systems and Humans}, 30\penalty0 (3):\penalty0 286--297, 2000.

\bibitem[Perdomo et~al.(2020)Perdomo, Zrnic, Mendler-D\"{u}nner, and
  Hardt]{perdomo2020performative}
Juan~C. Perdomo, Tijana Zrnic, Celestine Mendler-D\"{u}nner, and Moritz Hardt.
\newblock Performative prediction.
\newblock In \emph{Proceedings of the 37th International Conference on Machine
  Learning (ICML)}, pages 7599--7609, 2020.

\bibitem[Perez et~al.(2023)Perez, Ringer, Luko{\v{s}}i{\={u}}t{\.{e}}, Nguyen,
  Chen, Heiner, Pettit, Olsson, Kundu, Kadavath, et~al.]{perez2023discovering}
Ethan Perez, Sam Ringer, Kamil{\.{e}} Luko{\v{s}}i{\={u}}t{\.{e}}, Karina
  Nguyen, Edwin Chen, Scott Heiner, Craig Pettit, Catherine Olsson, Sandipan
  Kundu, Saurav Kadavath, et~al.
\newblock Discovering language model behaviors with model-written evaluations.
\newblock In \emph{Findings of the Association for Computational Linguistics:
  ACL 2023}, pages 13387--13434, 2023.

\bibitem[Platt(1999)]{platt1999probabilistic}
John Platt.
\newblock Probabilistic outputs for support vector machines and comparisons to
  regularized likelihood methods.
\newblock In \emph{Advances in Large Margin Classifiers}, pages 61--74. MIT
  Press, 1999.

\bibitem[Prelec(2004)]{prelec2004bayesian}
Dra\v{z}en Prelec.
\newblock A bayesian truth serum for subjective data.
\newblock \emph{Science}, 306\penalty0 (5695):\penalty0 462--466, 2004.

\bibitem[Proschan and Sethuraman(1977)]{proschan1977monotone}
Frank Proschan and Jayaram Sethuraman.
\newblock Schur functions in statistics {I}. the preservation theorem.
\newblock \emph{The Annals of Statistics}, 5\penalty0 (2):\penalty0 256--262,
  1977.

\bibitem[Rafailov et~al.(2023)Rafailov, Sharma, Mitchell, Ermon, Manning, and
  Finn]{rafailov2023direct}
Rafael Rafailov, Archit Sharma, Eric Mitchell, Stefano Ermon, Christopher~D.
  Manning, and Chelsea Finn.
\newblock Direct preference optimization: Your language model is secretly a
  reward model.
\newblock In \emph{Advances in Neural Information Processing Systems 36
  (NeurIPS 2023)}, 2023.

\bibitem[Rochet and Chon{\'e}(1998)]{rochet1998ironing}
Jean-Charles Rochet and Philippe Chon{\'e}.
\newblock Ironing, sweeping, and multidimensional screening.
\newblock \emph{Econometrica}, 66\penalty0 (4):\penalty0 783--826, 1998.

\bibitem[Saunders et~al.(2022)Saunders, Yeh, Wu, Bills, Ouyang, Ward, and
  Leike]{saunders2022self}
William Saunders, Catherine Yeh, Jeff Wu, Steven Bills, Long Ouyang, Jonathan
  Ward, and Jan Leike.
\newblock Self-critiquing models for assisting human evaluators.
\newblock \emph{arXiv preprint arXiv:2206.05802}, 2022.

\bibitem[Savage(1971)]{savage1971elicitation}
Leonard~J Savage.
\newblock Elicitation of personal probabilities and expectations.
\newblock \emph{Journal of the American Statistical Association}, 66\penalty0
  (336):\penalty0 783--801, 1971.

\bibitem[Scerri et~al.(2002)Scerri, Pynadath, and Tambe]{scerri2002towards}
Paul Scerri, David~V Pynadath, and Milind Tambe.
\newblock Towards adjustable autonomy for the real world.
\newblock \emph{Journal of Artificial Intelligence Research}, 17:\penalty0
  171--228, 2002.

\bibitem[Schelling(1960)]{schelling1960strategy}
Thomas~C. Schelling.
\newblock \emph{The Strategy of Conflict}.
\newblock Harvard University Press, Cambridge, MA, 1960.

\bibitem[Schervish(1989)]{schervish1989general}
Mark~J Schervish.
\newblock A general method for comparing probability assessors.
\newblock \emph{Annals of Statistics}, 17\penalty0 (4):\penalty0 1856--1879,
  1989.

\bibitem[Schick et~al.(2023)Schick, Dwivedi-Yu, Dess{\`\i}, Raileanu, Lomeli,
  Zettlemoyer, Cancedda, and Scialom]{schick2023toolformer}
Timo Schick, Jane Dwivedi-Yu, Roberto Dess{\`\i}, Roberta Raileanu, Maria
  Lomeli, Luke Zettlemoyer, Nicola Cancedda, and Thomas Scialom.
\newblock Toolformer: Language models can teach themselves to use tools.
\newblock In \emph{Advances in Neural Information Processing Systems},
  volume~36, 2023.

\bibitem[Schulman et~al.(2017)Schulman, Wolski, Dhariwal, Radford, and
  Klimov]{schulman2017proximal}
John Schulman, Filip Wolski, Prafulla Dhariwal, Alec Radford, and Oleg Klimov.
\newblock Proximal policy optimization algorithms.
\newblock \emph{arXiv preprint arXiv:1707.06347}, 2017.

\bibitem[Settles(2012)]{settles2012active}
Burr Settles.
\newblock \emph{Active Learning}.
\newblock Synthesis Lectures on Artificial Intelligence and Machine Learning.
  Morgan \& Claypool, 2012.

\bibitem[Sharma et~al.(2024)Sharma, Tong, Korbak, Duvenaud, Askell, Bowman,
  Cheng, Durmus, Hatfield-Dodds, Johnston, et~al.]{sharma2024towards}
Mrinank Sharma, Meg Tong, Tomasz Korbak, David Duvenaud, Amanda Askell,
  Samuel~R. Bowman, Newton Cheng, Esin Durmus, Zac Hatfield-Dodds, Scott~Riley
  Johnston, et~al.
\newblock Towards understanding sycophancy in language models.
\newblock In \emph{Proceedings of the 12th International Conference on Learning
  Representations (ICLR)}, 2024.

\bibitem[Shazeer et~al.(2017)Shazeer, Mirhoseini, Maziarz, Davis, Le, Hinton,
  and Dean]{shazeer2017outrageously}
Noam Shazeer, Azalia Mirhoseini, Krzysztof Maziarz, Andy Davis, Quoc Le,
  Geoffrey Hinton, and Jeff Dean.
\newblock Outrageously large neural networks: The sparsely-gated
  mixture-of-experts layer.
\newblock In \emph{Proceedings of the International Conference on Learning
  Representations}, 2017.

\bibitem[Tian et~al.(2023)Tian, Mitchell, Yao, Manning, and Finn]{tian2023just}
Katherine Tian, Eric Mitchell, Huaxiu Yao, Christopher~D Manning, and Chelsea
  Finn.
\newblock Just ask for calibration: Strategies for eliciting calibrated
  confidence scores from language models fine-tuned with human feedback.
\newblock In \emph{Proceedings of the 2023 Conference on Empirical Methods in
  Natural Language Processing}, pages 1--14, 2023.

\bibitem[Tirole(1986)]{tirole1986hierarchies}
Jean Tirole.
\newblock Hierarchies and bureaucracies: On the role of collusion in
  organizations.
\newblock \emph{Journal of Law, Economics, \& Organization}, 2\penalty0
  (2):\penalty0 181--214, 1986.

\bibitem[Tsybakov(2009)]{tsybakov2009introduction}
Alexandre~B. Tsybakov.
\newblock \emph{Introduction to Nonparametric Estimation}.
\newblock Springer, 2009.

\bibitem[Vickrey(1961)]{vickrey1961counterspeculation}
William Vickrey.
\newblock Counterspeculation, auctions, and competitive sealed tenders.
\newblock \emph{The Journal of Finance}, 16\penalty0 (1):\penalty0 8--37, 1961.

\bibitem[Vovk et~al.(2005)Vovk, Gammerman, and Shafer]{vovk2005algorithmic}
Vladimir Vovk, Alex Gammerman, and Glenn Shafer.
\newblock \emph{Algorithmic Learning in a Random World}.
\newblock Springer, New York, 2005.

\bibitem[Winkler(1969)]{winkler1969scoring}
Robert~L Winkler.
\newblock Scoring rules and the evaluation of probability assessors.
\newblock \emph{Journal of the American Statistical Association}, 64\penalty0
  (327):\penalty0 1073--1078, 1969.

\bibitem[Witkowski and Parkes(2012)]{witkowski2012peer}
Jens Witkowski and David~C Parkes.
\newblock Peer prediction without a common prior.
\newblock In \emph{Proceedings of the 13th ACM Conference on Electronic
  Commerce}, pages 964--981, 2012.

\bibitem[Witkowski et~al.(2023)Witkowski, Freeman, {Wortman Vaughan}, Pennock,
  and Krause]{witkowski2023incentive}
Jens Witkowski, Rupert Freeman, Jennifer {Wortman Vaughan}, David~M. Pennock,
  and Andreas Krause.
\newblock Incentive-compatible forecasting competitions.
\newblock \emph{Management Science}, 69\penalty0 (3):\penalty0 1354--1374,
  2023.
\newblock \doi{10.1287/mnsc.2022.4410}.

\bibitem[Wolf et~al.(2023)Wolf, Wies, Avnery, Levine, and
  Shashua]{wolf2023alignment}
Yotam Wolf, Noam Wies, Oshri Avnery, Yoav Levine, and Amnon Shashua.
\newblock Fundamental limitations of alignment in large language models.
\newblock \emph{arXiv preprint arXiv:2304.11082}, 2023.

\bibitem[Xie et~al.(2022)Xie, Raghunathan, Liang, and Ma]{xie2022explanation}
Sang~Michael Xie, Aditi Raghunathan, Percy Liang, and Tengyu Ma.
\newblock An explanation of in-context learning as implicit {B}ayesian
  inference.
\newblock In \emph{International Conference on Learning Representations}, 2022.
\newblock \doi{10.48550/arXiv.2111.02080}.

\end{thebibliography}

\end{document}